\def\year{2022}\relax
\documentclass[letterpaper]{article}
\usepackage{aaai22} 
\usepackage{times} 
\usepackage{helvet} 
\usepackage{courier} 
\usepackage[hyphens]{url} 
\usepackage{graphicx} 
\usepackage{graphicx} 
\usepackage{natbib} 
\usepackage{caption} 
\DeclareCaptionStyle{ruled}%
{labelfont=normalfont,labelsep=colon,strut=off}
\usepackage{amsfonts}       
\usepackage{nicefrac}       
\DeclareMathAlphabet\mathbfcal{OMS}{cmsy}{b}{n}
\usepackage{amsmath}
\usepackage{amsfonts} 
\usepackage{microtype}      
\usepackage{graphicx}
\usepackage{xcolor}         
\usepackage{amsmath}
\usepackage{amsmath}
\usepackage{float}
\usepackage{graphicx}
\usepackage{mathrsfs}
\usepackage{subfig}
\usepackage{amssymb}
\usepackage{bbm}
\usepackage{amsthm}
\newtheorem{theorem}{Theorem}
\newtheorem{lemma}[theorem]{Lemma}

\newtheorem{remark}{Remark}
\newtheorem{assumption}{$\mathbfcal{A}$}

\title{Gradient Temporal Difference with Momentum: Stability and Convergence}
\author{
Rohan Deb\thanks{Corresponding Author}, Shalabh Bhatnagar\\
}
\affiliations {
Department of Computer Science and Automation\\
Indian Institute of Science, \\
Bangalore\\
rohandeb@iisc.ac.in, shalabh@iisc.ac.in
}
\begin{document}
\maketitle
\begin{abstract}
Gradient temporal difference (Gradient TD) algorithms are a popular class of stochastic approximation (SA) algorithms used for policy evaluation in reinforcement learning. Here, we consider Gradient TD algorithms with an additional heavy ball momentum term and provide choice of step size and momentum parameter that ensures almost sure convergence of these algorithms asymptotically. In doing so, we decompose the heavy ball Gradient TD iterates into three separate iterates with different step sizes. We first analyze these iterates under one-timescale SA setting using results from current literature. However, the one-timescale case is restrictive and a more general analysis can be provided by looking at a three-timescale decomposition of the iterates. In the process we provide the first conditions for stability and convergence of general three-timescale SA. We then prove that the heavy ball Gradient TD algorithm is convergent using our three-timescale SA analysis. Finally, we evaluate these algorithms on standard RL problems and report improvement in performance over the vanilla algorithms.
\end{abstract}
\section{Introduction}
In reinforcement learning (RL), the goal of the learner or the agent is to maximize its long term accumulated reward by interacting with the environment. 
	One important task in most of RL algorithms is that of \textit{policy evaluation}. It predicts the average accumulated reward an agent would receive from a state (called \textit{value function}) if it follows the given policy. 
	In \textit{model-free learning}, the agent does not have access to the underlying dynamics of the environment and has to learn the \textit{value function} from samples of the form (state, action, reward, next-state).
	Two very popular algorithms in the \textit{model-free} setting are \textit{Monte-Carlo} (MC) and \textit{temporal difference} (TD) learning (see \citet{SuttonBarto_book}, \citet{Sutton1988}). 
	It is a well known fact that TD learning diverges in the off-policy setting (see \citet{Baird}). 
	 A class of algorithms called \textit{gradient temporal difference} (Gradient TD) were introduced in \cite{gtd} and \cite{FastGradient} which are convergent even in the off-policy setting. 
	These algorithms fall under a larger class of algorithms called linear stochastic approximation (SA) algorithms. 
	
	A lot of literature is dedicated to studying the asymptotic behaviour of SA algorithms starting from the work of \cite{robbins_munro}. 
	In recent times, the ODE  method to analyze asymptotic behaviour of SA \cite{Ljung, Kushner_book2,Borkar_Book,Borkar_Meyn} has become quite popular in the RL community. 
	The Gradient TD methods were shown to be convergent using the ODE approach. A generic one-timescale (One-TS) SA iterate has the following form:
	\begin{equation}
	    \label{Intro_OTS}
	    x_{n+1} = x_{n} + a(n)\left(h(x_n) + M_{n+1}\right),
	\end{equation}
	where $x \in \mathbb{R}^{d_1}$ are the iterates. The function $h:\mathbb{R}^{d_1}\rightarrow\mathbb{R}^{d_1}$ is assumed to be a Lipschitz continuous function. $M_{n+1}$ is a Martingale difference noise sequence and $a(n)$ is the step-size at time-step $n$. Under some mild assumptions, the iterate  given by \eqref{Intro_OTS} converges (see \citeauthor{Borkar_Book} \citeyear{Borkar_Book}; \citeauthor{Borkar_Meyn} \citeyear{Borkar_Meyn}). When $h$ is a linear map of the form $b-Ax_n$, the matrix $A$ is often called the driving matrix. The three Gradient TD algorithms: GTD \cite{gtd}, GTD2 and TDC \cite{FastGradient} consist two iterates of the following form:
	\begin{equation}
	    \label{Intro_TTS_1}
	    x_{n+1} = x_{n} + a(n)(h(x_n,y_n) + M_{n+1}^{(1)},
	\end{equation}
	\begin{equation}
	    \label{Intro_TTS_2}
	     y_{n+1} =  y_{n} + b(n)(g(x_n,y_n) + M_{n+1}^{(2)}),
	\end{equation}
    where $x \in \mathbb{R}^{d_1}$,  $y \in \mathbb{R}^{d_2}$. See section \ref{secPrelims} for exact form of the iterates. 
    The two iterates still form a One-TS SA scheme if \(\lim_{n\rightarrow\infty}\frac{b(n)}{a(n)} = c\), where $c$ is a constant and a two-timescale (two-TS) scheme if \(\lim_{n\rightarrow\infty}\frac{b(n)}{a(n)} = 0\).
    
	Separately, adding a momentum term to accelerate the convergence of iterates is a popular technique in stochastic gradient descent (SGD). The two most popular schemes are the Polyak's Heavy ball method \cite{polyak_heavy_ball}, and Nesterov's accelerated gradient method \cite{nesterov}. A lot of literature is dedicated to studying momentum with SGD. Some recent works include \cite{ghadimi2014global, Loizou, Gitman, Ma, Assran}. Momentum in the SA setting, which is the focus of the current work, has limited results. Very few works study the effect of momentum in the SA setting. A recent work by \cite{MJ} studies SA with momentum briefly and shows an improvement of mixing rate. However, the setting considered is restricted to linear SA and the driving matrix is assumed to be symmetric. Further, the iterates involve an additional Polyak-Ruppert averaging \cite{polyak_avg}. Here, in contrast, we analyze the asymptotic behaviour of the algorithm and make none of the above assumptions. A somewhat distant paper is by \cite{Devraj} that introduces Matrix momentum in SA and is not equivalent to heavy ball momentum. 
	
	A very recent work by \cite{webpage} studied One-TS SA with heavy ball momentum in the univariate case (i.e., $d = 1$ in iterate \eqref{Intro_OTS}) in the context of web-page crawling. The iterates took the following form:
	\begin{equation}
	    \label{Intro_OTS_mom}
	    x_{n+1} = x_{n} + a(n)\left(h(x_n) + M_{n+1}\right) + \eta_n(x_{n} - x_{n-1}).
	\end{equation}
	The momentum parameter $\eta_n$ was chosen to decompose the iterate into two recursions of the form given by \eqref{Intro_TTS_1} and \eqref{Intro_TTS_2}. 
	We use such a decomposition for Gradient TD methods with momentum. This leads to three separate iterates with three step-sizes. We analyze these three iterates and provide stability (iterates remain bounded throughout) and almost sure (a.s.) convergence guarantees.
	\subsection{Our Contribution}
	\label{subsection_Our_cont}
	\begin{itemize}
	    \item We first consider the One-TS decomposition of Gradient TD with momentum iterates and show that the driving matrix in this case is Hurwitz (all eigen values are negative). Thereafter we use the theory of One-TS SA to show that the iterates are stable and convergent to the same TD solution.
	    \item Next, we consider the Three-TS decomposition. We provide the first stability and convergence conditions for general Three-TS recursions. We then show that the iterates under consideration satisfy these conditions.
	    \item Finally, we evaluate these algorithms for different choice of step-size and momentum parameters on standard RL problems and report an improvement in performance over their vanilla counterparts.
	\end{itemize}
\section{Preliminaries}
\label{secPrelims}

%
In the standard RL setup, an agent interacts with the environment which is a Markov Decision Process (MDP). At each discrete time step $t$, the agent is in state $s_t \in \mathcal{S},$ takes an action $a_t \in \mathcal{A},$ receives a reward $r_{t+1} \equiv r(s_{t},a_t,s_{t+1}) \in \mathbb{R}$ and moves to another state $s_{t+1} \in \mathcal{S}$. Here $\mathcal{S}$ and $\mathcal{A}$ are finite sets of possible states and actions respectively. The transitions are governed by a kernel $\mathbb{P}$. A policy $\pi:\mathcal{S}\times\mathcal{A}\rightarrow[0,1]$ is a mapping that defines the probability of picking an action in a state. We let $P^{\pi}(s'|s)$ be the transition probability matrix induced by $\pi$. Also, $\{d^{\pi}(s)\}_{s \in \mathcal{S}}$ represents the steady-state distribution for the Markov chain induced by $\pi$ and the matrix $D$ is a diagonal matrix of dimension $n \times n$ with the entries $d^{\pi}(s)$ on its diagonals.The state-value function associated with a policy $\pi$ for state $s$ is 
\[V^{\pi}(s) = \mathbb{E}_{\pi}\left[ \sum_{t=0}^{\infty} \gamma^{t} R_{t+1}|s_{0} = s\right],\]
where $\gamma$ $\in [0,1)$ is the discount factor.

In the linear architecture setting, \textit{policy evaluation} deals with estimating $V^{\pi}(s)$ through a linear model $V_{\theta}(s) = \theta^{T}\phi(s)$, where $\phi(s) \equiv \phi_{s}$ is a feature associated with the state $s$ and $\theta$ is the parameter vector. We define the TD-error as \(\delta_{t} = r_{t+1} + \gamma \theta_{t}^{T}\phi_{t+1} - \theta_{t}^{T}\phi_{t}\) and $\Phi$ as an $n\times d$ matrix where the $s^{th}$ row is $\phi(s)^T$. In the i.i.d setting it is assumed that the tuple $(\phi_{t},\phi_{t}'$) (where $\phi_{t+1} \equiv \phi_{t}'$ ) is drawn independently from the stationary distribution of the Markov chain induced by $\pi$. Let \(\bar{A} = \mathbb{E}[\phi_{t}(\gamma\phi_{t}'-\phi_{t})^{T}]\) and $\bar{b} = \mathbb{E}[r_{t+1}\phi_t]$, where the expectations are w.r.t. the stationary distribution of the induced chain. The matrix $\bar{A}$ is negative definite (see \citet{Maei_PhD,TsitsiklisVanRoy}). In the off-policy case, the importance weight is given by $\rho_{t} = \frac{\pi(a_{t}|s_{t})}{\mu(a_{t}|s_{t})}$, where $\pi$ and $\mu$ are the target and behaviour policies respectively. Introduced in \cite{gtd}, Gradient TD are a class of TD algorithms that are convergent even in the off-policy setting. 
Next, we present the iterates associated with the algorithms GTD \cite{gtd}, GTD2, TDC \cite{FastGradient}.
\begin{itemize}
    \item \textbf{GTD}: 
    \begin{gather}
        \label{gtd_1}
        \theta_{t+1} = \theta_{t} + \alpha_{t}(\phi_{t} -\gamma\phi_{t}')\phi_{t}^{T}u_{t},\\
        \label{gtd_2}
        u_{t+1} = u_{t} + \beta_{t}(\delta_{t}\phi_{t} - u_{t}).
    \end{gather}
    \item \textbf{GTD2}: 
    \begin{gather}
        \label{gtd2_1}
        \theta_{t+1} = \theta_{t} + \alpha_{t}(\phi_{t} - \gamma\phi_{t}')\phi_{t}^{T}u_{t},\\
        \label{gtd2_2}
        u_{t+1} = u_{t} + \beta_{t}(\delta_{t} - \phi_{t}^{T}u_{t})\phi_{t}.
    \end{gather}
    \item \textbf{TDC}: 
    \begin{gather}
        \label{tdc_1}
        \theta_{t+1} = \theta_{t} + \alpha_{t}\delta_{t}\phi_{t} - \alpha_{t}\gamma \phi_{t}'(\phi_{t}^{T}u_{t}),\\
        \label{tdc_2}
        u_{t+1} = u_{t} + \beta_{t}(\delta_{t} - \phi_{t}^{T}u_{t})\phi_{t}.
    \end{gather}
\end{itemize}
The objective function for GTD is Norm of Expected Error defined as $NEU(\theta) = \mathbb{E}[\delta\phi]$. The GTD algorithm is derived by expressing the gradient direction as $-\frac{1}{2} \nabla NEU(\theta)$ = $\mathbb{E}\left[ (\phi-\gamma \phi')\phi^T\right] \mathbb{E}[\delta(\theta)\phi]$. Here $\phi' \equiv \phi(s')$. If both the expectations are sampled together, then the term would be biased by their correlation. An estimate of the second expectation is maintained as a long-term quasi-stationary estimate (see \eqref{gtd_1}) and the first expectation is sampled (see \eqref{gtd_2}). For GTD2 and TDC, a similar approach is used on the objective function Mean Square Projected Bellman Error defined as $MSPBE(\theta) = ||V_{\theta} - \Pi T^{\pi} V_{\theta}||_{D}$. Here, $\Pi$ is the projection operator that projects vectors to the subspace $\{\Phi\theta|\theta\in\mathbb{R}^{d}\}$ and $T^{\pi}$ is the Bellman operator defined as $T^{\pi}V = R^{\pi} + \gamma P^{\pi}V$. As originally presented, GTD and GTD2 are one-timescale algorithms ($\frac{\alpha_{t}}{\beta_{t}}$ is constant) while TDC is a two-timescale algorithm ($\frac{\alpha_{t}}{\beta_{t}} \rightarrow 0$). It was shown in all the three cases that $\theta_{n}\rightarrow\theta^{*} = -\bar{A}^{-1}\bar{b}$.
\section{Gradient TD with Momentum}
\label{secGTD}
Although, Gradient TD starts with a gradient descent based approach, it ends up with two-TS SA recursions. Momentum methods are known to accelerate the convergence of SGD iterates. Motivated by this, we examine momentum in the SA setting, and ask if the SA recursions for Gradient TD with momentum even converge to the same TD solution. We probe the heavy ball extension of the three Gradient TD algorithms where, we keep an accumulation of the previous gradient values in $\zeta_{t}$. Then, at time step $t+1$ the new gradient value multiplied by the step size is added to the current accumulation vector $\zeta_{t}$ multiplied by the momentum parameter $\eta_t$ as below:
\[\zeta_{t+1} = \eta_{t}\zeta_{t} + \alpha_{t}(\phi_{t}-\gamma\phi_{t}')\phi_{t}^{T}u_{t}.\]
The parameter $\theta$ is then updated in the negative of the direction $\zeta_{t+1}$, i.e.,
\(\theta_{t+1} = \theta_{t} - \zeta_{t+1}.\)
Since $u_{t+1}$ is computed as a long-term estimate of $\mathbb{E}[\delta(\theta)\phi]$, its update rule remains same.
The momentum parameter $\eta_t$ is usually set to a constant in the stochastic gradient setting. An exception to this can however be found in \cite{Gitman,StochasticHB}, where $\eta_{t} \rightarrow 1$. Here, we consider the latter case.  
Substituting $\zeta_{t+1}$ into the iteration of $\theta_{t+1}$ and noting that $\zeta_{t} = \theta_{t} - \theta_{t-1}$, the iterates for GTD with Momentum (\textbf{GTD-M}) can be written as:
\begin{equation}
\label{gtd_M_1}
    \theta_{t+1} = \theta_{t} + \alpha_{t}(\phi_{t} - \gamma\phi_{t}')\phi_{t}^{T}u_{t} + \eta_{t}(\theta_{t} - \theta_{t-1}),
\end{equation}
\begin{equation}
\label{gtd_M_2}
    u_{t+1} = u_{t} + \beta_{t}(\delta_{t}\phi_{t} - u_{t}).
\end{equation}
Similarly the iterates for \textbf{GTD2-M} are given by:
\setlength{\belowdisplayskip}{0pt}
\begin{equation}
\label{gtd2_M_1}
    \theta_{t+1} = \theta_{t} + \alpha_{t}(\phi_{t} - \gamma\phi_{t}')\phi_{t}^{T}u_{t} + \eta_{t}(\theta_{t} - \theta_{t-1}),
\end{equation}
\begin{equation}
\label{gtd2_M_2}
     u_{t+1} = u_{t} + \beta_{t}(\delta_{t} - \phi_{t}^{T}u_{t})\phi_{t}.
\end{equation}
Finally, the iterates for \textbf{TDC-M} are given by:
\begingroup
\setlength{\belowdisplayskip}{0pt}
\begin{equation}
\begin{split}
\label{tdc_M_1}
    \theta_{t+1} = \theta_{t} + \alpha_{t}(\delta_{t}\phi_{t} - \gamma\phi_{t}'(\phi_{t}^{T}u_{t}))+ \eta_{t}(\theta_{t} - \theta_{t-1}),
\end{split}
\end{equation}
\begin{equation}
\label{tdc_M_2}
     u_{t+1} = u_{t} + \beta_{t}(\delta_{t} - \phi_{t}^{T}u_{t})\phi_{t}.
\end{equation}
We choose the momentum parameter $\eta_{t}$ as in \cite{webpage} as follows: \(\eta_{t} = \frac{\varrho_{t}-w\alpha_{t}}{\varrho_{t-1}}\), where $\{\varrho_{t}\}$ is a positive sequence s.t. $\varrho_t\rightarrow0$ as $t\rightarrow\infty$ and $w \in \mathbb{R}$ is a constant. Note that $\eta_t\rightarrow 1$ as $t\rightarrow\infty$. We later provide conditions on $\varrho_t$ and $w$ to ensure a.s. convergence. As we would see in section \ref{section_CA}, the condition on $w$ in the One-TS setting is restrictive. Specifically, it depends on the norm of the driving matrix $\bar{A}$. This motivates us to look at the Three-TS setting and then the corresponding condition on $w$ is less restrictive. Using the momentum parameter as above,
\begin{equation*}
    \begin{split}
        \theta_{t+1} &= \theta_{t} + \alpha_{t}(\phi_{t} - \gamma\phi_{t}')\phi_{t}^{T}u_{t} +\frac{\varrho_{t} - w\alpha_{t}}{\varrho_{t-1}} (\theta_{t}-\theta_{t-1})
\end{split}
\end{equation*}
Rearranging the terms and dividing by $\rho_{t}$, we get:
\begin{equation*}
\begin{split}
        \frac{\theta_{t+1} - \theta_{t}}{\varrho_{t}} &= \frac{\theta_{t}-\theta_{t-1}}{\varrho_{t-1}}\\ 
        &+ \frac{\alpha_{t}}{\varrho_{t}}\Bigg((\phi_{t}- \gamma\phi_{t}')\phi_{t}^{T}u_{t} - w\left(\frac{\theta_{t} - \theta_{t-1}}{\varrho_{t-1}}\right)\Bigg).
    \end{split}
\end{equation*}
We let 
\[\frac{\theta_{t+1} - \theta_{t}}{\varrho_{t}} = v_{t+1}, \xi_{t} = \frac{\alpha_{t}}{\varrho_{t}} \mbox{ and } \varepsilon_{t} = v_{t+1} - v_{t}.\] 
Then, the GTD-M iterates in \eqref{gtd_M_1} and \eqref{gtd_M_2} can be re-written with the following three iterates:
\begin{gather}
    \label{GTD-M-1}
    v_{t+1} = v_{t} + \xi_{t}\left((\phi_{t} - \gamma\phi_{t}')\phi_{t}^{T}u_{t} - w v_{t}\right)\\
    \label{GTD-M-2}
    u_{t+1} = u_{t} + \beta_{t} (\delta_{t}\phi_{t} - u_{t})\\
    \label{GTD-M-3}
    \theta_{t+1} = \theta_{t} + \varrho_{t}(v_{t} + \varepsilon_{t})
\end{gather}
A similar decomposition can be done for the GTD2-M and TDC-M iterates.

\section{Convergence Analysis}
\label{section_CA}
In this section we analyze the asymptotic behaviour of the GTD-M iterates given by \eqref{GTD-M-1}, \eqref{GTD-M-2} and \eqref{GTD-M-3}. Throughout the section, we consider $v_t, u_t, \theta_t \in \mathbb{R}^{d}$. We first consider the One-TS case when $\beta_t = c_1 \xi_{t}$ and $\varrho_t = c_2 \xi_t$ $\forall t$, for some real constants $c_1, c_2>0$. Subsequently, we consider the Three-TS setting where $\frac{{\beta_t}}{\xi_{t}}\rightarrow0$ and $\frac{{\varrho_t}}{\beta_{t}}\rightarrow0$ as $t\rightarrow\infty$.

\subsection{One-Timescale Setting}
We begin by analyzing GTD-M using a one-timescale SA setting. We let $c_1 = c_2 = 1$ for simplicity. The iterates of GTD-M can then be re-written as:
\begin{equation}
\label{GTD-M-1TS}
    \psi_{t+1} = \psi_{t} + \xi_{t}(G_{t}\psi_{t} + g_{t} + \varepsilon_{t}),
\end{equation}
where, 
\begin{gather*}
\psi_{t} =
\begin{pmatrix}
    v_{t}\\
    u_{t}\\
    \theta_{t}
\end{pmatrix},
g_{t} = 
\begin{pmatrix}
    0\\
    r_{t+1}\phi_{t}\\
    0
\end{pmatrix},
\bar{\varepsilon}_{t} = 
\begin{pmatrix}
    0\\
    0\\
    \varepsilon_{t}
\end{pmatrix},
\end{gather*}
\begin{gather*}
G_{t} = 
\begin{pmatrix}
    -wI & (\phi_{t}-\gamma\phi_{t}')\phi_{t}^{T} & 0\\
    0 & -I & \phi_{t}(\gamma\phi_{t}'-\phi_{t})^{T}\\
    I &0 &0
\end{pmatrix}.
\end{gather*}
Equation \eqref{GTD-M-1TS} can be re-written in the general SA scheme as:
\begin{equation}
\label{GTD-M-1TS-general}
    \psi_{t+1} = \psi_{t} + \xi_{t}(h(\psi_{t}) + M_{t+1} + \bar{\varepsilon}_{t}).
\end{equation}
Here \(h(\psi) = g + G\psi, g = \mathbb{E}[g_{t}], G = \mathbb{E}[G_{t}]\), where the expectations are w.r.t. the stationary distribution of the Markov chain induced by the target policy $\pi$. $M_{t+1} = (G_{t+1}-G)\psi_{t} + (g_{t+1} - g)$. In particular,
\[
G = \begin{pmatrix}
-wI & -\bar{A}^{T} & 0\\
0 & -I & \bar{A} \\
I & 0 & 0
\end{pmatrix},
g = 
\begin{pmatrix}
0\\
\bar{b}\\
0
\end{pmatrix},
\]
where recall that \(\bar{A} = \mathbb{E}[\phi(\gamma\phi' - \phi)^{T}]\) and $\bar{b} = \mathbb{E}[r\phi]$
\begin{lemma}
\label{lemma-1}
Assume, $w(w+1) > ||\bar{A}||^2$. Then, the matrix $G$ is Hurwitz.
\end{lemma}
\begin{proof}
    Let $\lambda$ be an eigenvalue of $G$. The characteristic equation of the matrix $G$ is given by:
    \begin{gather*}
        \begin{vmatrix}
            -wI-\lambda I & -\bar{A}^{T}& 0\\
             0& -I-\lambda I& \bar{A}\\
             I& 0& -\lambda I
        \end{vmatrix} = 0\\
        \begin{vmatrix}
            wI+\lambda I & \bar{A}^{T}& 0\\
             0& I+\lambda I& -\bar{A}\\
             -I& 0& \lambda I
        \end{vmatrix} = 0
    \end{gather*}
Using the following formula for determinant of block matrices
\begin{gather*}
\begin{vmatrix}
    A_{11} & A_{12} & A_{13}\\
    A_{21} & A_{22} & A_{23}\\
    A_{31} & A_{32} & A_{33}
\end{vmatrix}=\\
\begin{vmatrix}
A_{11}
\end{vmatrix}
\begin{vmatrix}
\begin{pmatrix}
    A_{22} & A_{23}\\
    A_{32} & A_{33}
\end{pmatrix} -
\begin{pmatrix}
    A_{21}\\
    A_{31}
\end{pmatrix}A_{11}^{-1}
\begin{pmatrix}
    A_{12} & A_{13}
\end{pmatrix}
\end{vmatrix}
\end{gather*}
we have,
\begin{gather*}
        \begin{vmatrix}
            wI+\lambda I & \bar{A}^{T}& 0\\
             0& I+\lambda I& -\bar{A}\\
             -I& 0& \lambda I
        \end{vmatrix} = \\
        \begin{vmatrix}
            (w+\lambda)I
        \end{vmatrix}
        \begin{vmatrix}
            \begin{pmatrix}
                I+\lambda I& -\bar{A}\\
                0& \lambda I
            \end{pmatrix}
            -\frac{1}{w+\lambda}
            \begin{pmatrix}
                0\\
                -I
            \end{pmatrix}
            \begin{pmatrix}
                \bar{A}^{T} & 0
            \end{pmatrix}
        \end{vmatrix}\\
        = (w+\lambda)^{d} 
        \begin{vmatrix}
            I+\lambda I &-\bar{A}\\
            \frac{\bar{A}^{T}}{w+\lambda} & \lambda I
        \end{vmatrix}\\
        = (w+\lambda)^{d} 
        \begin{vmatrix}
            (1+\lambda)I
        \end{vmatrix}
        \begin{vmatrix}
            \lambda I + \frac{1}{(1+\lambda)(w+\lambda)}\bar{A}^{T}\bar{A}
        \end{vmatrix}\\
        = \frac{(w+\lambda)^{d}(1+\lambda)^{d}}{(w+\lambda)^{d}(1+\lambda)^{d}}
        \begin{vmatrix}
            \lambda(1+\lambda)(w+\lambda) I + \bar{A}^{T}\bar{A}
        \end{vmatrix}\\
        =\begin{vmatrix}
            \lambda(1+\lambda)(w+\lambda) I + \bar{A}^{T}\bar{A}
        \end{vmatrix}
\end{gather*}
Therefore, from the characteristic equation of $G$, we have that
\[
\begin{vmatrix}
    \lambda(1+\lambda)(w+\lambda) I + \bar{A}^{T}\bar{A}
\end{vmatrix} = 0.
\]
There must exist a non-zero vector $x\in \mathbb{C}^{d}$, such that 
\[
x^{*}(\lambda(1+\lambda)(w+\lambda) I + \bar{A}^{T}\bar{A})x = 0,
\]
where $x^{*}$ is the conjugate transpose of the vector $x$ and $x^{*}x = ||x||^{2}> 0$. The above equation reduces to the following cubic-polynomial equation:
\[\lambda^{3}||x||^{2} + (w+1)\lambda^{2}||x||^{2} + w\lambda||x||^{2} + ||\bar{A}x||^{2}=0,\]
where $||\bar{A}x||^{2} = x^{*}\bar{A}^{T}\bar{A}x$.
Using Routh-Hurwitz criterion, a cubic polynomial
\(a_{3}\lambda^{3} + a_{2}\lambda^{2} + a_{1}\lambda + a_{0}\) has all roots with negative real parts iff $a_{3},a_{2},a_{1},a_{0} > 0$ and $a_{1}a_{2} > a_{0}a_{3}$. In our case, \(a_{3} = ||x||^{2} > 0, a_{2} = (w+1)||x||^{2} > 0, a_{1} = w||x||^{2} > 0 \mbox{ and } a_{0} = ||\bar{A}x||^2>0\). The last inequality follows from the fact that $\bar{A}$ is negative definite and therefore \(x^{*}\bar{A}^{T}\bar{A}x > 0\). Finally, \(a_{1}a_{2} = w(w+1)||x||^{4}, a_{0}a_{3} = ||x||^{2}||\bar{A}x||^{2}\) and \(a_{1}a_{2} > a_{0}a_{3}\) follows from\( \frac{||\bar{A}x||^{2}}{||x||^{2}}<||\bar{A}||^2<w(w+1)\).
Therefore \(Re(\lambda)<0\) and the claim follows.
\end{proof}
Consider the following assumptions:
\begin{assumption}
\label{3A1}
All rewards $r(s,s')$ and features $\phi(s)$ are bounded, i.e., $r(s,s')\leq 1$ and $||\phi(s)||\leq 1$ $\forall s,s' \in \mathcal{S}$. Also, the matrix $\Phi$ has full rank, where $\Phi$ is an $n\times d$ matrix where the s$^{th}$ row is $\phi(s)^T$.
\end{assumption}
\begin{assumption}
\label{3A2}
The step-sizes satisfy \(\xi_{t} = \beta_{t} = \varrho_{t}>0\), \[\sum_{t}\xi_{t} = \infty \sum_{t}\xi_{t}^2 < \infty \mbox{ ,where } \xi_{t} = \frac{\alpha_{t}}{\varrho_{t}}\]
and the momentum parameter satisfies:
 \(\eta_{t} = \frac{\varrho_{t}-w\alpha_{t}}{\varrho_{t-1}}.\)
\end{assumption}
\begin{assumption}
\label{3A3}
The samples ($\phi_{t},\phi_{t}'$) are drawn i.i.d from the stationary distribution of the Markov chain induced by target policy $\pi$.
\end{assumption}
\begin{theorem}
    Assume $\mathbfcal{A}$\textbf{\textup{\ref{3A1}}}, $\mathbfcal{A}$\textbf{\textup{\ref{3A2}}} and $\mathbfcal{A}$\textbf{\textup{\ref{3A3}}} hold and let $w \geq 1$. Then, the GTD-M iterates given by \eqref{gtd_M_1} and \eqref{gtd_M_2} satisfy \(\theta_{n} \rightarrow \theta^{*} = -\bar{A}^{-1}\bar{b}\) a.s. as \(n\rightarrow \infty\). 
\end{theorem}
\begin{proof}
    Assumption $\mathbfcal{A}$\textbf{\textup{\ref{3A1}}} ensures that $||\bar{A}||^2<w(w+1)$ and $\mathbfcal{A}$\textbf{\textup{\ref{3A3}}} ensures that the function $h(\cdot)$ is well defined. Now, using Lemma \ref{lemma-1} and \cite{Borkar_Meyn} we can show that the iterates in \eqref{GTD-M-1TS} remain stable. Then using the third extension from (Chapter-2 pp. 17, \citet{Borkar_Book}) we can show that \(\psi_{n} \rightarrow -G^{-1}g\) as \(n\rightarrow\infty\). Thereafter using the formula for inverse of block matrices it can be shown that $\theta_{n}\rightarrow -\bar{A}^{-1}b$ as \(n\rightarrow\infty\). See Appendix A1 for a detailed proof.
\end{proof}
\noindent Similar results can be proved for the GTD2-M and TDC-M iterates. 
\begin{remark}
If $w$ is large, the initial values of the momentum parameter is small. The condition on $w$ in lemma \ref{lemma-1} is large compared to the condition on $w$ in \cite{webpage}, where the condition is $w>0$. Motivated by this, we look at the three-TS case of the iterates.
\end{remark} 

\subsection{Three Timescale Setting}
\label{sub_ThreeTS}
We consider the three iterates for GTD-M in \eqref{GTD-M-1}, \eqref{GTD-M-2} and \eqref{GTD-M-3} under the following criteria for step-sizes: $\frac{\xi_{t}}{{\beta_t}}\rightarrow0$ and $\frac{{\varrho_t}}{\xi_{t}}\rightarrow0$ as $t\rightarrow\infty$.
We provide the first conditions for stability and a.s. convergence of generic three-TS SA recursions. We emphasize that the setting we look at in Theorem \ref{theorem_3TS} is more general than the setting at hand of GTD-M iterates. Although stability and convergence results exist for one-TS and two-TS cases, this is the first time such results have been provided for the case of three-TS recursions. We next provide the general iterates for a three-TS recursion along with the assumptions used while analyzing them.
Consider the following three iterates:
\begin{equation}
    \label{main_iter_x}
    x_{n+1} = x_{n} + a(n)\left(h(x_{n},y_{n},z_{n}) + M_{n+1}^{(1)} + \varepsilon^{(1)}_n\right),
\end{equation}
\begin{equation}
    \label{main_iter_y}
    y_{n+1} = y_{n} + b(n)\left(g(x_{n},y_{n},z_{n}) + M_{n+1}^{(2)} + \varepsilon^{(2)}_n\right),
\end{equation}
\begin{equation}
    \label{main_iter_z}
    z_{n+1} = z_{n} + c(n)\left(f(x_{n},y_{n},z_{n}) + M_{n+1}^{(3)} + \varepsilon^{(3)}_n\right),
\end{equation}
and the following assumptions:
\begin{itemize}
    \item[\bf{(B1)}] $h:\mathbb{R}^{d_1+d_2+d_3}\rightarrow\mathbb{R}^{d_1}, g:\mathbb{R}^{d_1+d_2+d_3}\rightarrow\mathbb{R}^{d_2}, f:\mathbb{R}^{d_1+d_2+d_3}\rightarrow\mathbb{R}^{d_3}$ are Lipchitz continuous, with Lipchitz constants $L_1,L_2$ and $L_3$ respectively.
    
    \item[\bf{(B2)}] $\{a(n)\}$, $\{b(n)\}$, $\{c(n)\}$ are step-size sequences that satisfy $a(n)>0,b(n)>0,c(n)>0, \forall n>0,$
    \[\sum_{n}a(n) = \sum_{n}b(n) = \sum_{n}c(n) = \infty,\]
    \[\sum_{n}(a(n)^2+b(n)^2+c(n)^2) < \infty,\]
    \[\frac{b(n)}{a(n)}\rightarrow 0, \frac{c(n)}{b(n)}\rightarrow 0 \mbox{ as } n \rightarrow \infty.\]
    
    \item[\bf{(B3)}] $\{M_{n}^{(1)}\}, \{M_{n}^{(2)}\}, \{M_{n}^{(3)}\}$ are Martingale difference sequences w.r.t. the filtration $\{\mathcal{F}_{n}\}$ where,
    \[\mathcal{F}_{n} = \sigma\left(x_{m},y_{m},z_{m},M_{m}^{(1)},M_{m}^{(2)},M_{m}^{(3)},m\leq n\right)\]
    \[\mathbb{E}\left[||M_{n+1}^{(i)}||^{2}|\mathcal{F}_{n}\right] \leq K_i \left(1+||x_{n}||^{2}+||y_{n}||^{2}+||z_{n}||^{2}\right);\]
    $\forall n\geq0$, $i = 1,2,3$ and constants $0<K_i<\infty$. The terms $\varepsilon^{(i)}_t$ satisfy $ ||\varepsilon^{(1)}_n|| + ||\varepsilon^{(2)}_n|| + ||\varepsilon^{(3)}_n|| \rightarrow 0$ as $n\rightarrow \infty$.
    
    \item[\bf(B4)]
    \begin{enumerate}
        \item[(i)] The ode $\dot{x}(t) = h(x(t),y,z), y \in \mathbb{R}^{d_2}, z \in \mathbb{R}^{d_3}$ has a globally asymptotically stable equilibrium (g.a.s.e) $\lambda(y,z)$, and $\lambda:\mathbb{R}^{d_2\times d_3}\rightarrow \mathbb{R}^{d_1}$ is Lipchitz continuous.
        \item[(ii)] The ode $\dot{y}(t) = g(\lambda(y(t),z),y(t),z), z \in \mathbb{R}^{d_3}$ has a globally asymptotically stable equilibrium $\Gamma(z)$, where $\Gamma:\mathbb{R}^{d_3}\rightarrow \mathbb{R}^{d_2}$ is Lipchitz continuous.
        \item[(iii)] The ode $\dot{z}(t) = f(\lambda(\Gamma(z(t)),z(t)),\Gamma(z(t)),z(t))$, has a globally asymptotically stable equilibrium $z^{*}\in\mathbb{R}^{d_3}$.
    \end{enumerate}
    \item[\bf(B5)] The functions $h_c(x, y, z) = \frac{h(cx,cy,cz)}{c}, c\geq1$ satisfy $h_c\rightarrow h_{\infty}$ as $c\rightarrow\infty$ uniformly on compacts. The ODE: 
    \(\dot{x}(t) = h_{\infty}(x(t),y,z),\)
    has a unique globally asymptotically stable equilibrium $\lambda_{\infty}(y,z)$, where $\lambda_{\infty}:\mathbb{R}^{d_2+d_3}\rightarrow\mathbb{R}^{d_1}$ is Lipschitz continuous. Further, $\lambda_{\infty}(0,0)= 0$. 
    
    \item[\bf(B6)] The functions $g_c(y, z) = \frac{g(c\lambda_{\infty}(y,z),cy,cz)}{c}, c\geq1$ satisfy $g_c\rightarrow g_{\infty}$ as $c\rightarrow\infty$ uniformly on compacts. The ODE: 
    \(\dot{y}(t) = g_{\infty}(y(t),z),\)
    has a unique globally asymptotically stable equilibrium $\Gamma_{\infty}(z)$, where $\Gamma_{\infty}:\mathbb{R}^{d_3}\rightarrow\mathbb{R}^{d_2}$ is Lipschitz continuous. Further, $\Gamma_{\infty}(0) = 0$.
    
    \item[\bf(B7)] The functions $f_c(z) = \frac{g(c\lambda_{\infty}(\Gamma_{\infty}(z),z),c\Gamma_{\infty}(z),cz)}{c}, c\geq1$ satisfy $f_c\rightarrow f_{\infty}$ as $c\rightarrow\infty$ uniformly on compacts. The ODE: 
    \(\dot{z}(t) = f_{\infty}(z(t)),\)
    has the origin in $\mathbb{R}^{d_3}$ as its unique globally asymptotically stable equilibrium.
\end{itemize}
\begin{remark}
Conditions $\bf{(B5)} - \bf{(B7)}$ give sufficient conditions that ensure that the iterates remain stable. Specifically it ensures that  $\sup_{n} (||x_{n}|| + ||y_{n}|| + ||z_{n}||) < \infty$ $a.s.$. Conditions $\bf{(B1)} - \bf{(B4)}$ along with the stability of iterates ensures a.s. convergence of the iterates.
\end{remark}
\begin{theorem}
    \label{theorem_3TS}
    Under assumptions $\bf{(B1)}$-$\bf{(B7)}$,the iterates given by \eqref{main_iter_x} satisfy \eqref{main_iter_y} and \eqref{main_iter_z},
\[(x_n, y_n, z_n) \rightarrow (\lambda(\Gamma(z^{*}),z^{*}),\Gamma(z^{*}),z^{*})\mbox{ as } n\rightarrow\infty\]
\end{theorem}
\begin{proof}
    See Appendix A2.
\end{proof}
Next we use theorem \ref{theorem_3TS}, to show that the iterates of GTD-M a.s. converge to the TD solution $-\bar{A}^{-1}\bar{b}$. Consider the following assumption on step-size sequences instead of  $\mathbfcal{A}$\textbf{\textup{\ref{3A2}}}.
\begin{assumption}
\label{3A4}
The step-sizes satisfy \(\xi_{t}>0,\beta_{t}>0, \varrho_{t}>0 \mbox{ }\forall t\), 
\[ \sum_{t}\xi_t = \sum_{t}\beta_t = \sum_{t}\varrho_t = \infty,\]
    \[\sum_{t}(\xi_t^2+\beta_t ^2+\varrho_t^2) < \infty,\]
    \[\frac{\beta_t}{\xi_t}\rightarrow 0, \frac{\varrho_t}{\beta_t}\rightarrow 0 \mbox{ as } t \rightarrow \infty\]
and the momentum parameter satisfies:
 \(\eta_{t} = \frac{\varrho_{t}-w\alpha_{t}}{\varrho_{t-1}}.\)
\end{assumption}
\begin{theorem}
    \label{theorem_GTD-M_3TS}
     Assume $\mathbfcal{A}$\textbf{\textup{\ref{3A1}}}, $\mathbfcal{A}$\textbf{\textup{\ref{3A3}}} and $\mathbfcal{A}$\textbf{\textup{\ref{3A4}}} hold and let $w>0$. Then, the GTD-M iterates given by \eqref{gtd_M_1} and \eqref{gtd_M_2} satisfy \(\theta_{n} \rightarrow \theta^{*} = -\bar{A}^{-1}\bar{b}\) a.s. as \(n\rightarrow \infty\). 
\end{theorem}
\begin{proof}
We transform the iterates given by \eqref{GTD-M-1}, \eqref{GTD-M-2} and \eqref{GTD-M-3} into the standard SA form given by \eqref{main_iter_x}, \eqref{main_iter_y} and \eqref{main_iter_z}. Let $\mathcal{F}_t = \sigma(u_0, v_0, \theta_0, r_{j+1},\phi_j, \phi_j': j < t)$. Let, $A_t = \phi_t(\gamma\phi_t'-\phi_t)^T$ and $b_t = r_{t+1}\phi_t$. Then, \eqref{GTD-M-1} can be re-written as:
\begin{equation*}
    v_{t+1} = v_t + \xi_{t}\left(h(v_t,u_t,\theta_t) + M_{t+1}^{(1)}\right)
\end{equation*}
where,
\begin{equation*}
    \begin{split}
        h(v_t,u_t,\theta_t) &= \mathbb{E}[(\phi_t - \gamma\phi_t')\phi_t^Tu_t - w v_t|\mathcal{F}_t]\\ 
        &= -\bar{A}^Tu_t-wv_t .\\
        M_{t+1}^{(1)} = -A_t^Tu_t& - w v_t - h(v_t,u_t,\theta_t)  = (\bar{A}^T-A_t^T)u_t .
    \end{split}
\end{equation*}
Next, \eqref{GTD-M-2} can be re-written as:
\begin{equation*}
    \begin{split}
        u_{t+1} &= u_t + \beta_t\left(g(v_t,u_t,\theta_t) + M_{t+1}^{(2)}\right)\\
\end{split}
\end{equation*}
where,
\begin{equation*}
    \begin{split}
        g(v_t,u_t,\theta_t) &= \mathbb{E}[\delta_t\phi_t - u_t|\mathcal{F}_t]
        = \bar{A}\theta_t + \bar{b} - u_t\\
        M_{t+1}^{(2)} &= A_t\theta_t + b_t - u_t - g(v_t,u_t,\theta_t)\\ 
        &= (A_t-\bar{A})\theta_t + (b_t-\bar{b}).
    \end{split}
\end{equation*}
Finally, \eqref{GTD-M-3} can be re-written as:
\begin{equation*}
    \begin{split}
        \theta_{t+1} = \theta_t + \varrho_{t}\left(f(v_t,u_t,\theta_t) + \varepsilon_t + M_{t+1}^{(3)}\right)
    \end{split}
\end{equation*}
where,
\(
        f(v_t,u_t,\theta_t) = v_t \mbox{ and }
        M_{t+1}^{(3)} = 0.
\)
\begin{figure*}[!bt]
    \centering
    {
    {\includegraphics[width=0.9\linewidth]{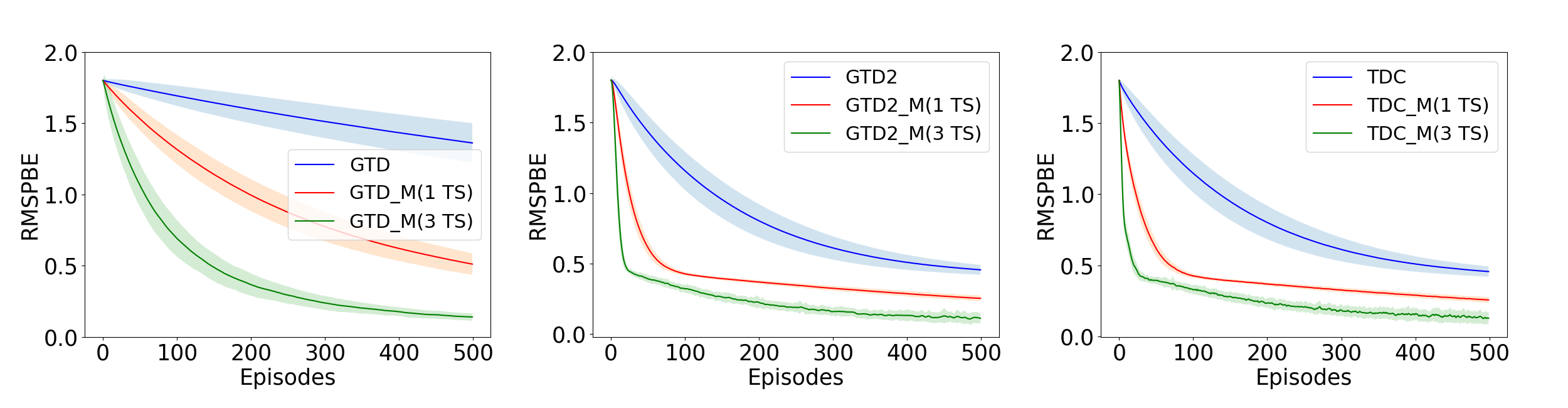}}
    
    \caption{RMSPBE (averaged over 100 independent runs) accross episodes for Boyan Chain. The features used are the standard spiked features of size 4 used in Boyan chain (see \cite{Dann}).}
    \label{f1}
    {\includegraphics[width=0.9\linewidth]{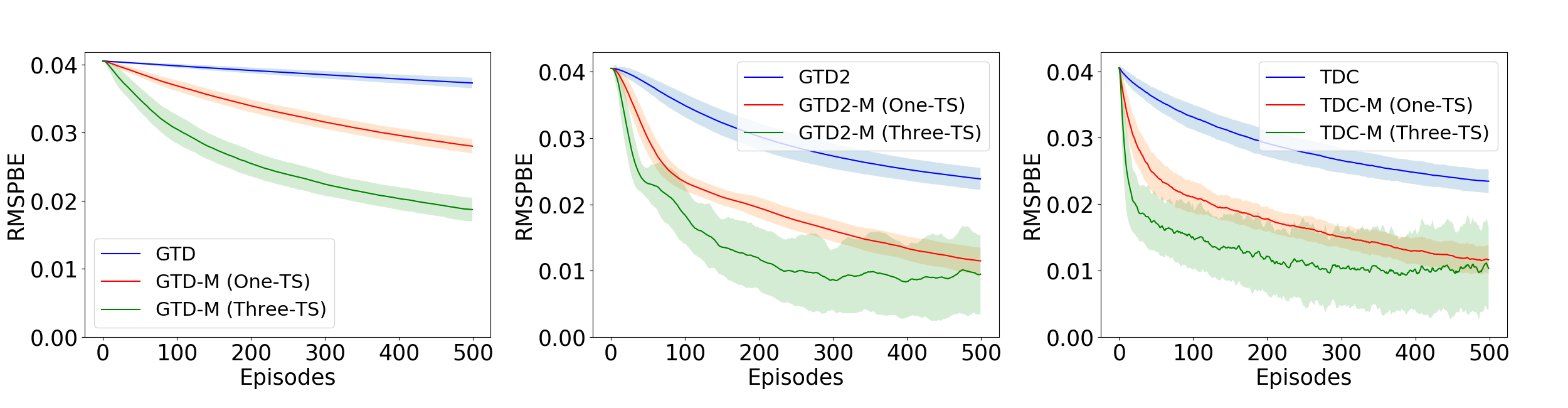}}
    \caption{RMSPBE (averaged over 100 independent runs) across episodes for the 5-State Random Chain problem. The features used are the \textit{Dependent} features used in \cite{FastGradient}.}
    \label{f2}
    }
\end{figure*}
The functions $h,g,f$ are linear in $v,u,\theta$ and hence Lipchitz continuous, therefore satisfying $\bf{(B1)}$. We choose the step-size sequences such that they satisfy $\bf{(B2)}$. One popular choice is \(\xi_t = \frac{1}{(t+1)^{\xi}}, \beta_t = \frac{1}{(t+1)^{\beta}}, \varrho_t= \frac{1}{(t+1)^{\varrho}},\)
\(\frac{1}{2}<\xi<\beta<\varrho\leq1.\)
Next, $M_{t+1}^{(1)},M_{t+1}^{(2)}$ and $M_{t+1}^{(3)}$ $t\geq0$, are martingale difference sequences w.r.t $\mathcal{F}_t$ by construction. $\mathbb{E}[||M_{t+1}^{(1)}||^2|\mathcal{F}_t] \leq ||(\bar{A}^T - A_t^T)||^2 ||u_t||^2$, $\mathbb{E}[||M_{t+1}^{(2)}||^2|\mathcal{F}_t] \leq 2(||(A_t-\bar{A})||^2 ||\theta_t||^2 + ||(b_t-\bar{b})||^2)$. The first part of $\bf{(B3)}$ is satisfied with $K_1 = ||(\bar{A}^T-A_{t}^T)||^2$, $K_2 = 2\max(||A_t - \bar{A}||^2,||b_t-\bar{b}||^2)$ and any $K_3>0$.
The fact that $K_1,K_2<\infty$ follows from the bounded features and bounded rewards assumption in $\mathbfcal{A}$\textbf{\textup{\ref{3A1}}}. Next, observe that $||\varepsilon_t^{(3)}||=\xi_t||\left((\phi_t - \gamma\phi_t')\phi_t^Tu_t - w v_t\right)||\rightarrow0$ since $\xi_t\rightarrow0 \mbox{ as } t\rightarrow\infty$. For a fixed $u,\theta \in \mathbb{R}^d$, consider the ODE
\(\dot{v}(t) = -\bar{A}^Tu - w v(t).\)
For $w>0$, $\lambda(u,\theta) = -\frac{\bar{A}^Tu}{w}$ is the unique g.a.s.e, is linear and therefore Lipchitz continuous. This satisfies $\bf{(B4)}$(i). Next, for a fixed $\theta \in \mathbb{R}^d$,
\(\dot{u}(t) = \bar{A}\theta + \bar{b} -u(t),\)
has $\Gamma(\theta) = \bar{A}\theta + \bar{b}$ as its unique g.a.s.e and is Lipschitz. This satisfies $\bf{(B4)}(ii)$. Finally, to satisfy $\bf{(B4)}(iii)$, consider,
\begin{equation*}
    \begin{split}
        \dot{\theta}(t) &= f(\lambda(\Gamma(z(t)),z(t)),\Gamma(z(t)),z(t))\\
       & = \frac{-\bar{A}^T\bar{A}\theta(t)-\bar{A}^T\bar{b}}{w}.
    \end{split}
\end{equation*}
Since, $\bar{A}$ is negative definite, therefore, $-\bar{A}^T\bar{A}$ is negative definite. Therefore, $\theta^* = -\bar{A}^{-1}\bar{b}$ is the unique g.a.s.e. 
Next, we show that the sufficient conditions for stability of the three iterates are satisfied. The function, $h_c(v,u,\theta) = \frac{-c\bar{A}^Tu-wcv}{c} = -\bar{A}^Tu-wv \rightarrow h_{\infty}(v,u,\theta) = -\bar{A}^Tu-wv$ uniformly on compacts as $c\rightarrow\infty$. The limiting ODE: 
\(\dot{v}(t) = -\bar{A}^Tu-wv(t)\)
has $\lambda_{\infty}(u,\theta) = -\frac{\bar{A}^Tu}{w}$ as its unique g.a.s.e. $\lambda_{\infty}$ is Lipschitz with $\lambda_{\infty}(0,0) = 0$, thus satisfying assumption $\bf{(B5)}$.

\noindent The function, $g_c(u,\theta) = \frac{c\bar{A}\theta + \bar{b} - cu}{c} = \bar{A}\theta-u+\frac{\bar{b}}{c} \rightarrow g_{\infty}(u,\theta) = -\bar{A}\theta - u$ uniformly on compacts as $c\rightarrow\infty$. The limiting ODE \(\dot{u}(t) = -\bar{A}\theta - u(t)\)
has $\Gamma_{\infty}(\theta) = \bar{A}\theta$ as its unique g.a.s.e. $\Gamma_{\infty}$ is Lipchitz with $\Gamma_{\infty}(0) = 0$. Thus assumption $\bf{(B6)}$ is satisfied.

\noindent Finally, $f_c(\theta) = \frac{-c\bar{A}^T\bar{A}\theta}{cw} \rightarrow f_{\infty} = \frac{-\bar{A}^T\bar{A}\theta}{w}$ uniformly on compacts as $c\rightarrow \infty$ and the ODE:
\(\dot{\theta}(t) = -\frac{\bar{A}^T\bar{A}\theta(t)}{w}\)
has origin as its unique g.a.s.e. This ensures the final condition $\bf{(B7)}$. By theorem \ref{theorem_3TS}, 
\[
\begin{pmatrix}
    v_t\\
    u_t\\
    \theta_t
\end{pmatrix}
\rightarrow
\begin{pmatrix}
    \lambda(\Gamma(-\bar{A}^{-1}\bar{b}),-\bar{A}^{-1}\bar{b})\\
    \Gamma(-\bar{A}^{-1}\bar{b})\\
    -\bar{A}^{-1}\bar{b}.
\end{pmatrix}
=
\begin{pmatrix}
    0\\
    0\\
    -\bar{A}^{-1}\bar{b}.
\end{pmatrix}
\]
Specifically, $\theta_t \rightarrow -\bar{A}^{-1}\bar{b}$.
\end{proof}
\noindent Similar analysis can be provided for GTD2-M and TDC-M iterates. See Appendix A3 for details.
\section{Experiments}
\label{sec_exp}
\begin{figure*}[!bt]
    \centering
    {
    {\includegraphics[width=0.9\linewidth]{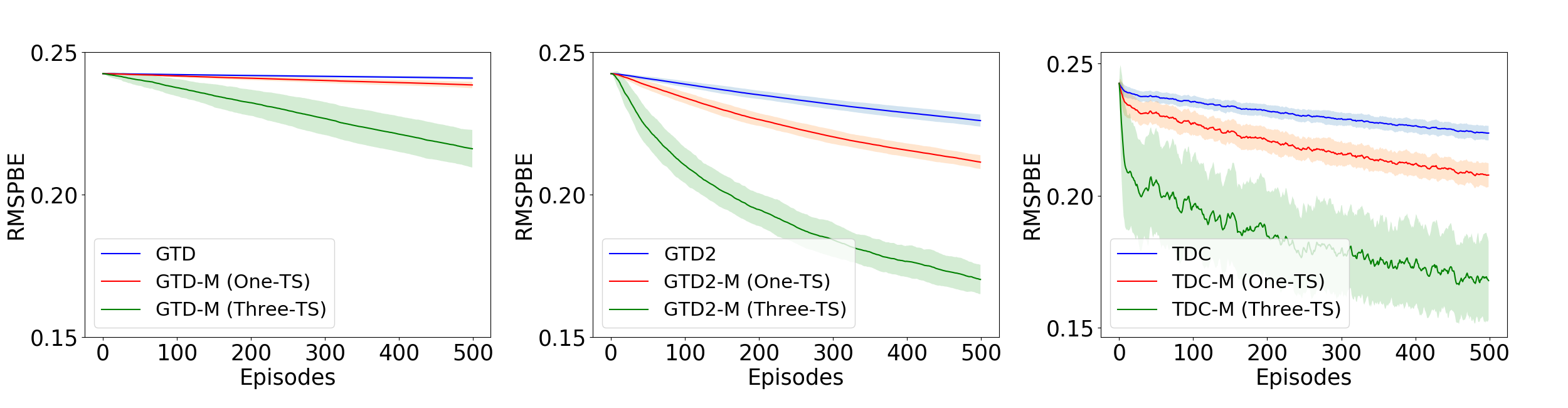}}
    \caption{RMSPBE (averaged over 100 independent runs) accross episodes for the 19-State Random Walk problem. The features used are an extension of the \textit{Dependent} features used in \cite{FastGradient}.}
    \label{f3}
    {\includegraphics[width=0.9\linewidth]{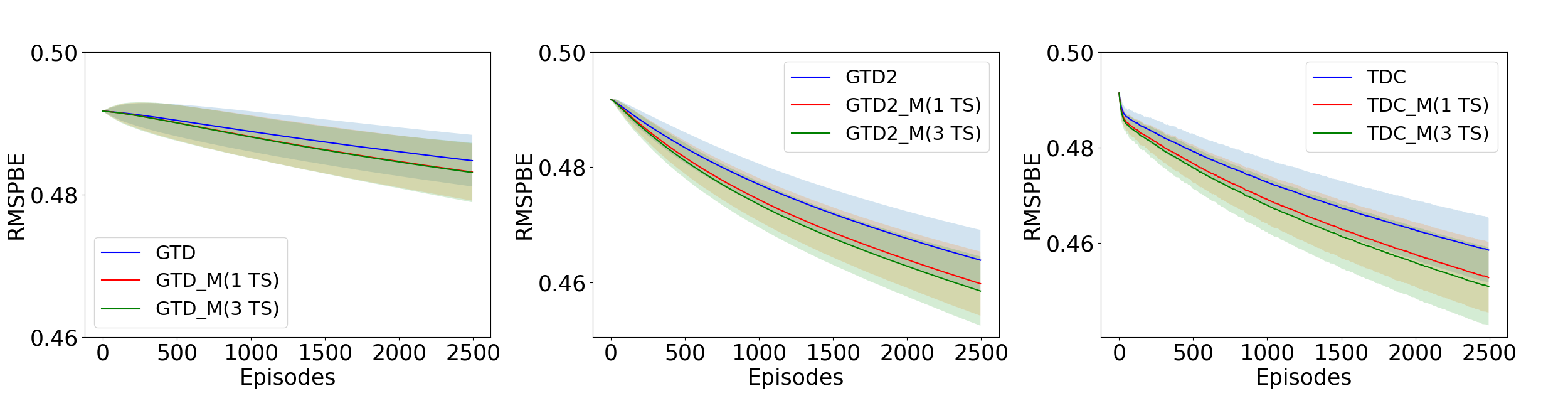}}
    \caption{RMSPBE (averaged over 100 independent runs) accross episodes for 20-state Random MDP with 5 random actions. The features used are \textit{Linear random} of size 10 (see \cite{Dann}). For each state, the value of the feature vector at $10^{th}$ position is 1 and all the values in all other 9 positions is chosen randomly from 0 to 10 and are then normalized.}
    \label{f4}
    }
\end{figure*}
We evaluate the momentum based GTD algorithms defined in section \ref{secGTD} to four standard problems of policy evaluation in reinforcement learning namely, Boyan Chain \cite{Boyan}, 5-State random walk \cite{FastGradient}, 19-State Random Walk \cite{SuttonBarto_book} and Random MDP \cite{FastGradient}. See Appendix A4 for a detailed description of the MDP settings and \cite{Dann} for details on implementation. We run the three algorithms, GTD, GTD2 and TDC along with their heavy ball momentum variants in One-TS and Three-TS settings and compare the RMSPBE (Root of MSPBE) across episodes. Figure-1 to Figure-4 plot these results. We consider decreasing step-sizes of the form:
\(\xi_t = \frac{1}{(t+1)^{\xi}},\beta_t = \frac{1}{(t+1)^{\beta}},\varrho_t = \frac{1}{(t+1)^{\varrho}}, \alpha_t = \frac{1}{(t+1)^{\alpha}}\)
in all the examples. Table \ref{stepsize-table} summarizes the different step-size sequences used in our experiment.

In one-TS setting, we require $\xi = \beta = \varrho$. Since $\xi_t = \frac{\alpha_t}{\varrho_t}$, we must have $\alpha = 2\varrho$. In the Three-TS setting, $\xi<\beta<\varrho$ thus implying, $\alpha<\varrho+\beta$ and $\beta<\varrho$. Although our analysis requires square summability: $\xi, \beta, \varrho >0.5$, such choice of step-size makes the algorithms converge very slowly.
Recently, \cite{dalal2018finite} showed convergence rate results for Gradient TD schemes with non-square summable
\\
\begin{table}
\caption{Choice of step-size parameters} 
\centering 
\begin{tabular}{c c c c c} 
\hline\hline 
Boyan Chain & $\alpha$ & $\beta$ & $\varrho$ &$w$\\ 
\hline 
Vanilla & 0.25 & 0.125 & - & - \\ 
One-TS & 0.25 & 0.125 & 0.125 & 1 \\
Three-TS & 0.25 & 0.125 & 0.2 & 0.1\\ 
\hline\hline 
5-state RW & $\alpha$ & $\beta$ & $\varrho$ &$w$\\ 
\hline 
Vanilla & 0.25 & 0.125 & - & - \\ 
One-TS & 0.25 & 0.125 & 0.125 & 1 \\
Three-TS & 0.25 & 0.125 & 0.2 & 0.1\\
\hline\hline 
19-State RW & $\alpha$ & $\beta$ & $\varrho$ &$w$\\ 
\hline 
Vanilla & 0.125 & 0.0625 & - & - \\ 
One-TS & 0.125 & 0.0625 & 0.0625 & 1 \\
Three-TS & 0.125 & 0.0625 & 0.1 & 0.1\\
\hline\hline 
Random Chain & $\alpha$ & $\beta$ & $\varrho$ &$w$\\ 
\hline 
Vanilla & 0.5 & 0.25 & - & - \\ 
One-TS & 0.5 & 0.25 & 0.25 & 1 \\
Three-TS & 0.5 & 0.25 & 0.3 & 0.1\\ 
\hline\hline 
\end{tabular}
\label{stepsize-table} 
\end{table}
\\
step-sizes also (See Remark 2 of \cite{dalal2018finite}).
Therefore, we look at non-square summable step-sizes here, and observe that in all the examples the iterates do converge. The momentum parameter is chosen as in $\mathbfcal{A}$\textbf{\textup{\ref{3A2}}}.

In all the examples considered, the momentum methods outperform their vanilla counterparts. Since, in the Three-TS setting, a lower value of $w$ can be chosen, this ensures that the momentum parameter is not small in the initial phase of the algorithm as in the One-TS setting. This in turn helps to reduce the RMSPBE faster in the initial phase of the algorithm as is evident from the experiments.
\section{Related Work and Conclusion}
\label{secConclusion}
To the best of our knowledge no previous work has specifically looked at Gradient TD methods with an added heavy ball term. The use of momentum specifically in the SA setting is very limited. Section 4.1 of \cite{MJ} does talk about momentum; however the problem looked at is that of SGD with momentum and the driving matrix is assumed to be symmetric (see Appendix H of their paper). We do not make any such assumption here. The work of \cite{Devraj}, indeed looks at momentum in SA setting. However, they introduce a matrix momentum term which is not equivalent to heavy ball momentum. Acceleration in Gradient TD methods has been looked at in \cite{acceleratedGTD}. The authors provide a new algorithm called ATD and the acceleration is in form of better data efficiency. However, they do not make use of momentum methods.

In this work we have introduced heavy ball momentum in Gradient Temporal difference algorithms for the first time. We decompose the two iterates of these algorithms into three separate iterates and provide asymptotic convergence guarantees of these new schemes under the same assumptions made by their vanilla counterparts. Specifically, we show convergence in the One-TS regime as well as Three-TS regime. In both the cases, the momentum parameter gradually goes 1. Three-TS formulation gives us more flexibility in choosing the momentum parameter. Specifically, compared to the One-TS setting, a larger momentum parameter can be chosen during the initial phase in the Three-TS case. We observe improved performance with these new schemes when compared with the original algorithms.

As a step forward from this work, the natural direction would be to look at more sophisticated momentum methods such as Nesterov's accelerated method \cite{nesterov}. Also, here we only provide the convergence guarantees of these new momentum methods. A particularly interesting step would be to quantify the benefits of using momentum in SA settings. Specifically, it would be interesting to extend weak convergence rate analysis of \cite{konda,mokkadem} to Three-TS regime. Also, extending the recent convergence rate results in expectation and high probability of GTD methods \cite{2TSgugan,harsh,Kaledin,tale2TS} to these momentum settings would be interesting works for the future.
\bibliography{References}
\onecolumn
\begin{center}
    \Huge\textbf{Appendix}
\end{center}
\par\noindent\rule{\textwidth}{0.1pt}
\section*{A1 \quad Proof of Theorem 2}
Consider the One timescale recursion for the GTD-M iterates given by \eqref{GTD-M-1TS-general} as given below:
\begin{equation*}
    \psi_{t+1} = \psi_{t} + \xi_{t}(h(\psi_{t}) + M_{t+1} + \bar{\varepsilon}_{t}).
\end{equation*}

Here \(h(\psi) = g + G\psi, g = \mathbb{E}[g_{t}], G = \mathbb{E}[G_{t}]\), where the expectations are w.r.t. the stationary distribution of the Markov chain induced by the target policy $\pi$. $M_{t+1} = (G_{t+1}-G)\psi_{t} + (g_{t+1} - g)$. In particular,
\[
G = \begin{pmatrix}
-wI & -\bar{A}^{T} & 0\\
0 & -I & \bar{A} \\
I & 0 & 0
\end{pmatrix},
g = 
\begin{pmatrix}
0\\
\bar{b}\\
0
\end{pmatrix},
\]
where recall that \(\bar{A} = \mathbb{E}[\phi(\gamma\phi' - \phi)^{T}]\) and $\bar{b} = \mathbb{E}[r\phi]$
We show that the conditions ${\bf(A1)-(A4)}$ in Chapter 2 of \cite{Borkar_Book} hold and thereafter use Theorem 2 of \cite{Borkar_Book} to show convergence to the TD solution.
\begin{itemize}
    \item[{\bf(A1)}] The map $h(\psi)$ is linear in $\psi$ and therefore Lipschitz continuous with Lipschitz constant $||G||$.
    \item[{\bf(A2)}] The step-size sequence $\xi_t$ satisfies the required conditions (cf. assumption $\mathbfcal{A}$ \textbf{\textup{\ref{3A2}}} of the current paper).
    \item[{\bf(A3)}] By construction $M_{t+1}$ is a martingale difference sequence w.r.t the filtration $\mathcal{F}_t = \sigma(\psi_{0},M_{k},k\leq t)$. Also, \(\mathbb{E}[||(G_{t+1} - G)\psi_t + (g_{t+1} - g)||^2|\mathcal{F}_t] \leq 2(||(G_{t+1} - G)||^2||\psi_t||^2 + ||g_{t+1}-g||^2)\). (A3) is satisfied with $K=2\max(||(G_{t+1} - G)||^2+||g_{t+1}-g||^2)$. $K<\infty$ follows from the fact that the rewards are uniformly bounded and the features are normalized (see assumption $\mathbfcal{A}$ \textbf{\textup{\ref{3A1}}}).
    \item[\bf{(A4)}] To ensure (A4), we show that (A5) of (Chapter 3, pp.22, \citet{Borkar_Book}) holds and then use Theorem 7 of \cite{Borkar_Book}. The functions $h_c(x) = \frac{h(cx)}{c} = \frac{g}{c} + G\psi, c\geq 1$. For any compact set $H$, $h_c \rightarrow h_{\infty}$ as $c\rightarrow\infty = G\psi$ uniformly. Consider the ODE 
    \[\dot{\psi}(t) = h_{\infty}(\psi(t)) = G\psi_t.\]
    Observe that since $||\phi_t||\leq1$ and $r_t \leq 1$ $\forall t$, we have $||A||^2<2$. Since we have assumed that $w\geq1$ therefore, $w(w+1)\geq||A||^2$, and hence from lemma \ref{lemma-1}, we have that $G$ is Hurwitz. Hence, the origin is a unique globally asymptotically stable equilibrium (g.a.s.e) for the above ODE. This in turn implies that the iterates remain bounded i.e., $\sup_t||\psi_t||<\infty$ a.s. $\forall t$. By (Theorem 2, Chapter 2 of \citet{Borkar_Book}) $\psi_t$ converges to an internally chain transitive invariant set of the ODE $\dot{\psi}(t) = h(\psi(t)) = g + G \psi(t)$. The only such point of the ODE is its equilibrium point $-G^{-1}g$. By (Corollary 4, Chapter 2 of \citet{Borkar_Book}),
    \[\psi_t \rightarrow -G^{-1}g.\]
    A straightforward calculation for the inverse of the $3\times3$ block matrix $G$ gives us that 
    \[\theta_t \rightarrow -\bar{A}^{-1}\bar{b}\]
\end{itemize}

\section*{A2 \quad Proof of Theorem 3}
\noindent We first start by assuming that the iterates remain stable (cf. assumption \textbf{(B5)}) and show that the three timescale recursions converge. Subsequently we provide conditions which ensure that the iterates remain stable. We consider general three timescale recursions as given below:
\begin{equation}
    \label{main_iter_x_app}
    x_{n+1} = x_{n} + a(n)\left(h(x_{n},y_{n},z_{n}) + M_{n+1}^{(1)}\right),
\end{equation}
\begin{equation}
    \label{main_iter_y_app}
    y_{n+1} = y_{n} + b(n)\left(g(x_{n},y_{n},z_{n}) + M_{n+1}^{(2)}\right),
\end{equation}
\begin{equation}
    \label{main_iter_z_app}
    z_{n+1} = z_{n} + c(n)\left(f(x_{n},y_{n},z_{n}) + M_{n+1}^{(3)}\right),
\end{equation}
where $x_n\in\mathbb{R}^{d_1}$, $y_n\in\mathbb{R}^{d_2}$ and $z_n\in\mathbb{R}^{d_3}$ $\forall n\geq 0$. Next we consider the following assumptions:
\begin{itemize}
    \item[\bf{(B1)}] $h:\mathbb{R}^{d_1+d_2+d_3}\rightarrow\mathbb{R}^{d_1}, g:\mathbb{R}^{d_1+d_2+d_3}\rightarrow\mathbb{R}^{d_2}, f:\mathbb{R}^{d_1+d_2+d_3}\rightarrow\mathbb{R}^{d_3}$ are Lipschitz continuous.
    \item[\bf{(B2)}] $\{M_{n}^{(1)}\}, \{M_{n}^{(2)}\}, \{M_{n}^{(3)}\}$ are Martingale difference sequences w.r.t. $\{\mathcal{F}_{n}\}$ where,
    \[\mathcal{F}_{n} = \sigma\left(x_{m},y_{m},z_{m},M_{m}^{(1)},M_{m}^{(2)},M_{m}^{(3)};m\leq n\right), n\geq0\]
    \[\mathbb{E}\left[||M_{n+1}^{(i)}||^{2}|\mathcal{F}_{n}\right] \leq K_i \left(1+||x_{n}||^{2}+||y_{n}||^{2}+||z_{n}||^{2}\right) a.s.; i = 1,2,3,\] for some constants $K_i>0$, $i=1,2,3$.
    \item[\bf{(B3)}] $\{a(n)\}$, $\{b(n)\}$, $\{c(n)\}$ are step-size sequences that satisfy $a(n)>0,b(n)>0,c(n)>0, \forall n\geq0$
    \[\sum_{n}a(n) = \sum_{n}b(n) = \sum_{n}c(n) = \infty,\]
    \[\sum_{n} (a(n)^2+b(n)^2+c(n)^2) < \infty,\]
    \[\frac{b(n)}{a(n)}\rightarrow 0 \mbox{,\quad} \frac{c(n)}{b(n)}\rightarrow 0 \mbox{ as } n \rightarrow \infty.\]
    \item[\bf(B4)]
    \begin{enumerate}
        \item[(i)] The ODE $\dot{x}(t) = h(x(t),y,z), y \in \mathbb{R}^{d_2}, z \in \mathbb{R}^{d_3}$ has a globally asymptotically stable equilibrium $\lambda(y,z)$, where $\lambda:\mathbb{R}^{d_2\times d_3}\rightarrow \mathbb{R}^{d_1}$ is Lipschitz continuous.
        \item[(ii)] The ODE $\dot{y}(t) = g(\lambda(y(t),z),y(t),z), z \in \mathbb{R}^{d_3}$ has a globally asymptotically stable equilibrium $\Gamma(z)$, where $\Gamma:\mathbb{R}^{d_3}\rightarrow \mathbb{R}^{d_2}$ is Lipschitz continuous.
        \item[(iii)] The ODE $\dot{z}(t) = f(\lambda(\Gamma(z(t)),z(t)),\Gamma(z(t)),z(t))$, has a globally asymptotically stable equilibrium $z^{*}\in\mathbb{R}^{d_3}$
    \end{enumerate}
    \item[\bf(B5)] $\sup_{n} \left(||x_{n}|| + ||y_{n}|| + ||z_{n}||\right) < \infty$ $a.s.$ 
\end{itemize}
\begin{theorem}
\label{theorem1}
Under $(\textbf{B1}) - (\textbf{B5})$ the iterates given by \eqref{main_iter_x_app}, \eqref{main_iter_y_app} and \eqref{main_iter_z_app},
\[(x_n, y_n, z_n) \rightarrow (\lambda(\Gamma(z^{*}),z^{*}),\Gamma(z^{*}),z^{*}).\]
\end{theorem}
\begin{proof}
We start with the following Lemma that characterizes the set to which the iterates converge. 
\begin{lemma}
\label{lemma1}
$(x_n, y_n, z_n) \rightarrow \{(\lambda(\Gamma(z),z),\Gamma(z),z): z\in\mathbb{R}^{d_3}\}$
\end{lemma}
\begin{proof}
We first consider the fastest timescale of $\{a(n)\}$ and show that:
\[(x_n, y_n, z_n)\rightarrow\{(\lambda(y, z),y,z):y \in \mathbb{R}^{d_2},z \in R^{d_3}\}.\]
We rewrite the iterates \eqref{main_iter_x_app}, \eqref{main_iter_y_app} and \eqref{main_iter_z_app} as:
\begin{equation}
    x_{n+1} = x_{n} + a(n)\left(h(x_{n},y_{n},z_{n}) + M_{n+1}^{(1)}\right),
\end{equation}
\begin{equation}
    y_{n+1} = y_{n} + a(n)\left(\epsilon^{(2),a}_n + M_{n+1}^{(2),a}\right),
\end{equation}
\begin{equation}
    z_{n+1} = z_{n} + a(n)\left(\epsilon^{(3),a}_n + M_{n+1}^{(3),a}\right),
\end{equation}
where, 
\[\epsilon^{(2),a}_n = \frac{b(n)}{a(n)}g(x_n,y_n,z_n),  M_{n+1}^{(2),a} = \frac{b(n)}{a(n)} M_{n+1}^{(2)},\]
\[\epsilon^{(3),a}_n = \frac{c(n)}{a(n)}f(x_n,y_n,z_n),  M_{n+1}^{(3),a} = \frac{c(n)}{a(n)} M_{n+1}^{(3)}.\]
Using third extension from Chapter-2 of \cite{BorkarBook}, $(x_n, y_n, z_n)$ converges to an internally chain transitive invariant set of the ODE 
\[\dot{x}(t) = h(x(t),y(t),z(t)),\]
\[\dot{y}(t) = 0,\]
\[\dot{z}(t) = 0.\]
For initial conditions $x \in \mathbb{R}^{d_1}, y \in \mathbb{R}^{d_2}, z \in \mathbb{R}^{d_3}$, the internally chain transitive invariant set of the above ODE is $\{(\lambda(y,z),y,z)\}$. Therefore,
\begin{equation}
\label{lemma1-i}
    (x_n,y_n,z_n)\rightarrow\{(\lambda(y,z),y,z):y\in\mathbb{R}^{d_2},z\in\mathbb{R}^{d_3}\}
\end{equation}
Next we consider the middle timescale $\{b(n)\}$. \eqref{main_iter_y_app} and \eqref{main_iter_z_app} can be re-written as:
\[y_{n+1} = y_n + b(n)\left(g(x_n,y_n,z_n)+M_{n+1}^{(2)}\right),\]
\begin{equation}
\label{middle_z}
    z_{n+1} = z_n + b(n)\left(\epsilon_n^{(3),b} + M_{n+1}^{(3),b}\right),
\end{equation}
where,
\[\epsilon_n^{(3),b} = \frac{b(n)}{c(n)}f(x_n,y_n,z_n),\mbox{\quad} M_{n+1}^{(3),b} = \frac{b(n)}{c(n)}M_{n+1}^{(3)}.\]
The iteration for $\{y_n\}$ can be re-written as:
\begin{equation}
\label{middle_y}
    y_{n+1} = y_n + b(n)\left(g(\lambda(y_n,z_n),y_n,z_n) + \epsilon_n^{(2),b} + M_{n+1}^{(2)}\right)
\end{equation}
where, 
\[\epsilon_n^{(2),b} = g(x_n,y_n,z_n) - g(\lambda(y_n,z_n),y_n,z_n).\]
Since, $(x_n,y_n,z_n) \rightarrow \{(\lambda(y,z),y,z):y \in \mathbb{R}^{d_2}, z \in \mathbb{R}^{d_3}\}$, therefore $||\epsilon_n^{(2),b}||\rightarrow0$ as $n\rightarrow\infty$. Again using third extension from Chapter-2 of \cite{BorkarBook}, it can be seen that \eqref{middle_z} and \eqref{middle_y} converges to an internally chain transitive invariant set of the ODE
\[\dot{y}(t) = g\left(\lambda(y(t),z(t)),y(t),z(t)\right)\]
\[\dot{z}(t) = 0\]
For initial conditions $y \in \mathbb{R}^{d_2}, z \in \mathbb{R}^{d_3}$, the internally chain transitive invariant set of the above ODE is $\{(\Gamma(z),z)\}$. Therefore,
\begin{equation}
\label{lemma1-ii}
    (y_n,z_n)\rightarrow\{(\Gamma(z),z):z\in\mathbb{R}^{d_3}.\}
\end{equation}
Combining \eqref{lemma1-i} and \eqref{lemma1-ii} we get:
\[(x_n,y_n,z_n)\rightarrow\{(\lambda(\Gamma(z),z),\Gamma(z),z):z\in\mathbb{R}^{d_3}\}.\]
\end{proof}
\noindent Finally, we consider the slowest timescale of $\{c(n)\}$. We define the piece wise linear continuous interpolation of the iterates $z_n$ as:
\[\bar{z}(t(n)) = z_n\]
\[\bar{z}(t) = z_n + (z_{n+1}-z_n)\frac{t-t(n)}{t(n+1)-t(n)}, t\in [t(n),t(n+1)],\]
where, \(t(n) = \sum_{m=0}^{n-1}c(n), n\geq1.\)
Also, let $z^{s}(t),t\geq s$, denote the unique solution to the below ODE starting at $s\in\mathbb{R}$:
\[\dot{z}^{s}(t) = h(z^{s}(t)), t\geq s,\]
with  $z^s(s) = \bar{z}(s)$. Using the arguments as in Theorem-2, Chapter-6 of \cite{BorkarBook}, it can be shown that for any $T>0$
\[\lim_{s\rightarrow\infty}\sup_{t\in[s,s+T]}||\bar{z}(t) - z^{s}(t)|| = 0 \mbox{ a.s. }\]
Subsequently arguing as in proof of Theorem-2, Chapter-2 of \cite{BorkarBook}, we get:
\[z_n\rightarrow z^* \mbox{ a.s. }\]
Using Lemma \ref{lemma1}, we get:
\[(x_n, y_n, z_n)\rightarrow \left(\lambda(\Gamma(z^*),z^*),\Gamma(z^*),z^*\right)\mbox{ a.s.}\]
\end{proof}
Next we provide sufficient conditions for $\bf(B5)$ to hold. Consider the following additional assumptions:
\begin{itemize}
    \item[\bf(B6)] The functions $h_c(x, y, z) \triangleq \frac{h(cx,cy,cz)}{c}, c\geq1$ satisfy $h_c\rightarrow h_{\infty}$ as $c\rightarrow\infty$ uniformly on compacts. For fixed $y\in\mathbb{R}^{d_2}, z\in\mathbb{R}^{d_3}$, the ODE
    \[\dot{x}(t) = h_{\infty}(x(t),y,z)\]
    has its unique globally asymptotically stable equilibrium $\lambda_{\infty}(y,z)$, where $\lambda_{\infty}:\mathbb{R}^{d_2+d_3}\rightarrow\mathbb{R}^{d_1}$ is Lipschitz continuous. Further, $\lambda_{\infty}(0,0) = 0$, i.e., 
    \[\dot{x}(t) = h_{\infty}(x(t),0,0)\]
    has origin in $\mathbb{R}^{d_1}$ as unique globally asymptotically stable equilibrium.
    
    \item[\bf(B7)] The functions $g_c(y, z) \triangleq \frac{g(c\lambda_{\infty}(y,z),cy,cz)}{c}, c\geq1$ satisfy $g_c\rightarrow g_{\infty}$ as $c\rightarrow\infty$ uniformly on compacts. For fixed $z\in\mathbb{R}^{d_3}$, the ODE
    \[\dot{y}(t) = g_{\infty}(y(t),z)\]
    has its unique globally asymptotically stable equilibrium $\Gamma_{\infty}(z)$, where $\Gamma_{\infty}:\mathbb{R}^{d_3}\rightarrow\mathbb{R}^{d_2}$ is Lipschitz continuous. Further, $\Gamma_{\infty}(0) = 0$, i.e., 
    \[\dot{y}(t) = g_{\infty}(y(t),0)\]
    has origin in $\mathbb{R}^{d_2}$ as its unique globally asymptotically stable equilibrium.
    
    \item[\bf(B8)] The functions $f_c(z) \triangleq \frac{f(c\lambda_{\infty}(\Gamma_{\infty}(z),z),c\Gamma_{\infty}(z),cz)}{c}, c\geq1$ satisfy $f_c\rightarrow f_{\infty}$ as $c\rightarrow\infty$ uniformly on compacts. The ODE
    \[\dot{z}(t) = f_{\infty}(z(t))\]
    has the origin in $\mathbb{R}^{d_3}$ as its unique globally asymptotically stable equilibrium.
\end{itemize}
\begin{theorem}
\label{thm_app_stability}
    Under assumptions $\bf{(B1)}$-$\bf{(B4)}$ and $\bf{(B6)}$-$\bf{(B8)}$,
    \[\sup_{n}(||x_n|| + ||y_n|| + ||z_n||) < \infty.\]
\end{theorem}
\begin{proof}
We begin with the fastest time scale determined by the step size $a(n)$. Consider the following definitions:
\begin{enumerate}
    \item[\bf{(F1)}] Define
    \[t(n) = \sum_{i=0}^{n-1}a(i), n\geq 1, \mbox{ with } t(0) = 0.\]
    Let $\psi_k = (x_k,y_k,z_k), k\geq0$, and 
    \[\bar{\psi}(t) = \psi_{n} + \left(\psi_{n+1} - \psi_{n}\right)\frac{t-t(n)}{t(n+1) - t(n)}, \mbox{ \quad }t \in [t(n), t(n+1)].\]
    
    \item[\bf{(F2)}]Given $t(n), n\geq 0$ and a constant $T>0$ define
    \[T_0 = 0,\]
    \[T_n = \min(t(m):t(m)\geq T_{n-1}+T), n\geq 1.\]
    One can find a subsequence $\{m(n)\}$ such that $T_n = t(m(n))$ $\forall n$ and $m(n) \rightarrow \infty$ as $n\rightarrow \infty$.
    
    \item[\bf{(F3)}] The scaling sequence is defined as:
    \[r(n) = \max(r(n-1),||\bar{\psi}(T_n)||,1),n\geq1.\]
    
    \item[\bf{(F4)}] The scaled iterates for $m(n)\leq k\leq m(n+1)-1$ are:
    \[\hat{x}_{m(n)} = \frac{x_{m(n)}}{r(n)},\mbox{\quad} \hat{y}_{m(n)} = \frac{y_{m(n)}}{r(n)},\mbox{\quad} \hat{z}_{m(n)} = \frac{z_{m(n)}}{r(n)},\]
    \[\hat{x}_{k+1} = \hat{x}_{k} + a(k)\left(\frac{h(c\hat{x}_{k},c\hat{y}_{k},c\hat{z}_{k})}{c} + \hat{M}_{k+1}^{(1)}\right)\]
    \[\hat{y}_{k+1} = \hat{y}_{k} + a(k)\left(\epsilon_{k}^{(2),a} + \hat{M}_{k+1}^{(2)}\right)\]
    \[\hat{z}_{k+1} = \hat{z}_{k} + a(k)\left(\epsilon_{k}^{(3),a} + \hat{M}_{k+1}^{(3)}\right)\]
    where, $c = r(n)$,
    \[\epsilon_{k}^{(2),a} = \frac{b(k)}{a(k)}\frac{g(c\hat{x}_k,c\hat{y}_{k},c\hat{z}_k)}{c}\]
    \[\epsilon_{k}^{(3),a} = \frac{c(k)}{a(k)}\frac{f(c\hat{x}_k,c\hat{y}_{k},c\hat{z}_k)}{c}\]
    \[\hat{M}_{k+1}^{(1)} = \frac{M_{k+1}^{(1)}}{r(n)}, \hat{M}_{k+1}^{(2)} = \frac{b(k)}{a(k)}\frac{M_{k+1}^{(2)}}{r(n)},\hat{M}_{k+1}^{(3)} = \frac{c(k)}{a(k)}\frac{M_{k+1}^{(3)}}{r(n)}.\]
    
    \item[\bf{(F5)}] Next we define the linearly interpolated trajectory for the scaled iterates as follows:
    \[\hat{\psi}(t) = \hat{\psi}_n + (\hat{\psi}_{n+1} - \hat{\psi}_n)\frac{t - t(n)}{t(n+1) - t(n)}, \mbox{\quad} t \in [t(n), t(n+1)].\]
    
    \item[\bf{(F6)}] Let \(\psi_{n}(t) = (x_n(t),y_n(t),z_n(t)), \mbox{\quad} t\in [T_n,T_{n+1}]\) denote the trajectory of the ODE:
    \[\dot{x}(t) = h_{r(n)}(x(t),y(t),z(t)),\]
    \[\dot{y}(t) = 0,\]
    \[\dot{z}(t) = 0,\]
    with $x_n(T_n) = \hat{x}(T_n)$, $y_n(T_n) = \hat{y}(T_n)$ and $z_n(T_n) = \hat{z}(T_n)$.
\end{enumerate}  
    First we state four lemmas for ODEs with two external inputs. The proofs of these lemmas follow exactly as Lemmas 2, 3, 4 and 5 of \cite{chandru-SB}. Subsequently when we analyze the middle timescale (timescale of $\{b(n)\}$) and slow timescale (timescale of $\{c(n)\}$) recursions, we restate the corresponding lemmas for ODEs with one and no external inputs respectively.
    Let $x_{c}^{y(t),z(t)}(t,x)$ and $x_{\infty}^{y(t),z(t)}(t,x)$ denote the solution to the ODEs \[\dot{x}(t) = h_c(x(t),y(t),z(t)), \mbox{\quad}t\geq 0,\]
    \[\dot{x}(t) = h_{\infty}(x(t),y(t),z(t)), \mbox{\quad}t\geq 0,\] respectively,
    with initial condition $x\in \mathbb{R}^{d_1}$ and the external inputs $y(t) \in \mathbb{R}^{d_2}$ and $z(t) \in \mathbb{R}^{d_3}$. Throughout the paper, $B(x,r) \triangleq \{q\in\mathbb{R}^{d_1}\Big|||q-x||<r\}, B(y,r) \triangleq \{q\in\mathbb{R}^{d_2}\Big|||q-y||<r\}$ and $B(z,r) \triangleq \{q\in\mathbb{R}^{d_3}\Big|||q-z||<r\}$ denote the ball of radius $r$ around $x,y$ and $z$ respectively.
    \begin{lemma}
    \label{Aux_a_1}
        Let $K \subset \mathbb{R}^{d_1}$ be a compact set, $y\in\mathbb{R}^{d_2}$ and $z\in \mathbb{R}^{d_3}$ be fixed external inputs. Then under $\bf{(B6)}$, given $\delta>0$, $\exists T_{\delta}>0$ such that $\forall x \in K$
        \[x^{y,z}_{\infty}(t,x) \in B(\lambda_{\infty}(y,z),\delta), \forall \delta\geq T_{\delta}.\]
    \end{lemma}
    \begin{lemma}
    \label{Aux_a_2}
        Let $x\in \mathbb{R}^{d_1}$, $y\in \mathbb{R}^{d_2}$, $z\in \mathbb{R}^{d_3}$, $[0,T]$ be a given time interval and $r>0$. Let $y'(t) \in B(y,r), $ $z'(t) \in B(z,r)$ $\forall t \in [0,T]$, then 
        \[||x_c^{y'(t),z'(t)}(t,x) - x_{\infty}^{y,z}(t,x)|| \leq (\epsilon(c) + 2Lr)Te^{LT}, \mbox{\quad}\forall t \in [0,T],\]
        where $\epsilon(c) \rightarrow 0$ as $c \rightarrow \infty$.
    \end{lemma}
    \begin{lemma}
    \label{Aux_a_3}
        Let $y\in \mathbb{R}^{d_2}$, $z\in \mathbb{R}^{d_3}$ then given $\epsilon >0$ and $T>0$, $\exists c_{\epsilon,T}>0$, $\delta_{\epsilon,T}>0$ and $r_{\epsilon,T}>0$ such that $\forall t \in [0,T)$, $\forall x \in B(\lambda_{\infty}(y,z),\delta_{\epsilon,T})$ $\forall c > c_{\epsilon,T}$ and external inputs $y'(s) \in B(y,r_{\epsilon,T})$ and $z'(s) \in B(z,r_{\epsilon,T})$. Then,
        \[x_c^{y'(t),z'(t)}(t,x) \in B(\lambda_{\infty}(y,z),2\epsilon) \mbox{\quad} \forall t \in [0,T].\]
    \end{lemma}
    \begin{lemma}
    \label{Aux_a_4}
        Let $x \in B(0,1) \subset \mathbb{R}^{d_1}, y\in K'\subset\mathbb{R}^{d_2}, z \in K'' \subset \mathbb{R}^{d_3}$ and let $\bf{(B6)}$ hold. Then given $\epsilon>0, \exists c_{\epsilon}\geq 1, r_{\epsilon}>0$ and $T_{\epsilon}>0$ such that for any external input satisfying \(y'(s) \in B(y,r_{\epsilon})\), \(z'(s) \in B(z,r_{\epsilon})\), $\forall s \in [0,T],$
        \[||x_c^{y'(t),z'(t)}(t,x) - \lambda_{\infty}(y,z)|| \leq 2\epsilon, \mbox{\quad} \forall c > c_{\epsilon},t\geq T_{\epsilon}.\]
    \end{lemma}
\noindent The next lemma uses the convergence result of three time scale iterates under the stability assumption of $\bf{(B5)}$ (Theorem \ref{theorem1}) and shows that the scaled iterates defined in $\bf{(F4)}$ converge.   
\begin{lemma}
    \label{fast1}
    Under $\bf{(B1)}-\bf{(B3)}$,
    \begin{enumerate}
        \item[(i)] For \(0\leq k \leq m(n+1)-m(n)\),
        \(||\hat{\psi}(t(m(n)+k))|| \leq K^{(1)}\)  a.s. for some constant $K^{(1)}>0.$
        
        \item[(ii)] \(\lim_{n\rightarrow\infty}||\hat{\psi}(t) - \psi_n(t)|| = 0 \mbox{ a.s. }\forall t \in [T_n,T_{n+1}]\)
    \end{enumerate}
\end{lemma}
\begin{proof}
    \begin{enumerate}
        \item[(i)] Follows as in (Lemma 4, Chapter-3, pp. 24, \citet{BorkarBook}).
        \item[(ii)] By construction, the iterates $\hat{x}_k$, $\hat{y}_k$, $\hat{z}_k$ remain bounded, i.e., $\sup_{k} (||\hat{x}_k|| + ||\hat{y}_k|| + ||\hat{z}_k||) < \infty$ a.s. Therefore, $\bf{(B1)}$-$\bf{(B4)}$ are satisfied. Using Theorem \ref{theorem1}, the iterates $(\hat{x}_n,\hat{y}_n,\hat{z}_n)$ converges. Using the third extension from Chapter-2 of \cite{BorkarBook}, the iterates $(\hat{x}_n,\hat{y}_n,\hat{z}_n)$ track the ODE system 
        \[\dot{x}(t) = h_{r(n)}(x(t),y(t),z(t)),\]
        \[\dot{y}(t) = 0,\]
        \[\dot{z}(t) = 0.\]
        Therefore, \(\lim_{n\rightarrow\infty}||\hat{\psi}(t) - \psi_n(t)|| = 0 \mbox{ a.s. }\forall t \in [T_n,T_{n+1}]\)
    \end{enumerate}
\end{proof}
In particular, Lemma \ref{fast1}(i) shows that along the fastest timescale between instants $T_n$ and $T_{n+1}$, the norm of the scaled iterate can grow at most by a factor $K^{(1)}$ starting from $B(0,1)$. Next, Lemma \ref{fast1}(ii) shows that the scaled iterate asymptotically tracks the ODE defined in $\bf{(F6)}$.
The next theorem bounds $||x_n||$ in terms of $||y_n||$ and $||z_n||$. We define the linearly interpolated trajectory of the three iterates as: $\forall t \in [t(n), t(n+1)]$,
    \[\bar{x}(t) = x_{n} + \left(x_{n+1} - x_{n}\right)\frac{t-t(n)}{t(n+1) - t(n)},\]
    
    \[\bar{y}(t) = y_{n} + \left(y_{n+1} - y_{n}\right)\frac{t-t(n)}{t(n+1) - t(n)},\]
    
    \[\bar{z}(t) = z_{n} + \left(z_{n+1} - z_{n}\right)\frac{t-t(n)}{t(n+1) - t(n)}.\]

\begin{theorem}
\label{fast_thm}
Under assumptions $\bf{(B1)}$-$\bf{(B4)}$ and $\bf{(B6)}$,
    \begin{enumerate}
        \item[(i)] For $n$ large, and $T = T_{\frac{1}{4}}$ (here $T$ is the sampling frequency as in \textbf{(F2)} and $T_{\frac{1}{4}}$ is $T_{\epsilon}$ as in Lemma \ref{Aux_a_4} with $\epsilon=\frac{1}{4}$), if 
        \(||\bar{x}(T_n)|| > C_a(1 + ||\bar{y}(T_n)|| + ||\bar{z}(T_n)||)\), for some $C_a>0$
        then \(||\bar{x}(T_{n+1})||\leq \frac{3}{4}||\bar{x}(T_n)||\)
        
        \item[(ii)] \(||\bar{x}(T_n)|| \leq C_a^{*}(1+ ||\bar{y}(T_n)||+ ||\bar{z}(T_n)||)\) a.s. for some $C_a^*>0$.
        
        \item[(iii)] \(||x_n|| \leq K_a^*(1+||y_n||+||z_n||), \mbox{ for some } K_a^*>0\)
    \end{enumerate}
\end{theorem}
\begin{proof}
\begin{enumerate}
    \item[(i)] We have \(||\bar{x}(T_n)|| > C_a(1 + ||\bar{y}(T_n)|| + ||\bar{z}(T_n)||).\) Since, \(r(n) = \max(r(n-1),||\bar{\psi}(T_n)||,1)\), this implies \(r(n)\geq||\bar{\psi}(T_n)||\). Therefore, $r(n)\geq C_a$. Next we show that \[||\hat{y}(T_n)||<\frac{1}{C_a} \mbox{ and }||\hat{z}(T_n)||<\frac{1}{C_a}.\]
    For $p\geq1$, 
    \begin{equation*}
        \begin{split}
            ||\hat{y}(T_n)||_p &= \frac{||\bar{y}(T_n)||_p}{r(n)}\leq\frac{||\bar{y}(T_n)||_p}{||\bar{\psi}(T_n)||_p}\\
            &=\frac{||\bar{y}(T_n)||_p}{\Big(||\bar{x}(T_n)||_p^p + ||\bar{y}(T_n)||_p^p + ||\bar{z}(T_n)||_p^p\Big)^{\frac{1}{p}}}
        \end{split}
    \end{equation*}
    Since, \(||\bar{x}(T_n)||_p \geq C_a(1 + ||\bar{y}(T_n)||_p + ||\bar{z}(T_n)||_p)\),
    \begin{equation}
        \begin{split}
            ||\bar{x}(T_n)||_p^p &\geq C_a^p\Big(||\bar{y}(T_n)||_p + ||\bar{z}(T_n)||_p\Big)^p\\
            &\geq C_a^p\Big(||\bar{y}(T_n)||_p^p + ||\bar{z}(T_n)||_p^p\Big)\\
        \end{split}
    \end{equation}
    Therefore,
    \begin{equation*}
        \begin{split}
            ||\hat{y}(T_n)||&\leq\frac{||\bar{y}(T_n)||_p}{\Big(C_a^p+1\Big)^\frac{1}{p}\Big(||\bar{y}(T_n)||_p^p+ ||\bar{z}(T_n)||_p^p\Big)^{\frac{1}{p}}}\\
            &\leq \frac{1}{\Big(1+C_a^p\Big)^{\frac{1}{p}}}\\
            & < \frac{1}{C_a}.
        \end{split}
    \end{equation*}
The second inequality follows from the fact that \(||\bar{y}(T_n)||_p^p \leq ||\bar{y}(T_n)||_p^p + ||\bar{z}(T_n)||_p^p\). A similar analysis proves \(||\hat{z}(T_n)||<\frac{1}{C_a}\). Next we show that 
\[||\hat{x}(T_n)||_p>\frac{1}{1+\frac{1}{C_a}}.\]
Here we are considering the case when iterates are blowing up. Therefore let $r(n) = \bar{\psi}(T_n)$. Then,
\begin{equation*}
\begin{split}
    ||\hat{x}(T_n)|| &= \frac{||\bar{x}(T_n)||}{||\bar{\psi}(T_n)||}\\
    &=\frac{||\bar{x}(T_n)||}{\Big(||\bar{x}(T_n)||_p^p + ||\bar{y}(T_n)||_p^p + ||\bar{z}(T_n)||_p^p\Big)^{\frac{1}{p}}}\\
    &= \frac{1}{\Big(1+\frac{||\bar{y}(T_n)||_p^p+||\bar{z}(T_n)||_p^p}{||\bar{x}(T_n)||_p^p}\Big)^{\frac{1}{p}}}\\
    &> \frac{1}{\Big(1+\frac{||\bar{y}(T_n)||_p^p+||\bar{z}(T_n)||_p^p}{C_a^p(||\bar{y}(T_n)||_p^p+||\bar{z}(T_n)||_p^p)}\Big)^{\frac{1}{p}}}\\
    & > \frac{1}{1+\frac{1}{C_a}}.
\end{split}
\end{equation*}
Let $y'(t-T_n) = y_n(t)$ and $z'(t-T_n) = z_n(t)$ $\forall t \in [T_n, T_{n+1}]$. From lemma \ref{Aux_a_4}, $\exists r_{\frac{1}{4}}, c_{\frac{1}{4}}, T_{\frac{1}{4}}$ such that 
\[||x_{c}^{y'(t),z'(t)}(t,\hat{x}(T_n))||\leq \frac{1}{4}, \forall t \geq T_{\frac{1}{4}}, \forall c \geq c_{\frac{1}{4}},\]
whenever $y'(t) \in B(0,r_{\frac{1}{4}})$ and $z'(t) \in B(0,r_{\frac{1}{4}})$. Choose $C_a > \max(c_{\frac{1}{4}},\frac{2}{r_{\frac{1}{4}}})$ and $T = T_{\frac{1}{4}}$. Since $\dot{y}(t) = 0,$ and $\dot{z}(t) = 0$ for the ODE defined in $\textbf{(F6)}, y'(t-T_n) = y_n(t) = \hat{y}(T_n)$ and $z'(t-T_n) = z_n(t) = \hat{z}(T_n)$ $\forall t \in [T_{n},T_{n+1}].$ From $||\hat{y}(T_n)|| < \frac{1}{C_a}$ and $||\hat{z}(T_n)|| < \frac{1}{C_a}$, it follows that $y'(s) \in B(0,r_{\frac{1}{4}})$ and $z'(s) \in B(0,r_{\frac{1}{4}})$ $\forall s \in [0,T].$ Using Lemma \ref{fast1}(ii), $||\hat{x}(T_{n+1}^{-}) - x_{n}(T_{n+1})|| < \frac{1}{4}$ for large enough $n$. Also observe that $||x_{n}(T_{n+1})|| = ||x^{y'(t),z'(t)}_{r(n)}(T_{n+1} - T_{n}, \hat{x}(T_{n}))|| \leq \frac{1}{4}$. Using these, we have $||\hat{x}(T_{n+1}^{-})|| \leq ||\hat{x}(T_{n+1}^{-})-x_{n}(T_{n+1})|| + ||x_{n}(T_{n+1})|| \leq \frac{1}{2}$. Finally since \[\frac{||\bar{x}(T_{n+1})||}{||\bar{x}(T_{n})||} = \frac{||\hat{x}(T_{n+1}^{-})||}{||\hat{x}(T_n)}||,\] we have 
\begin{equation*}
    \begin{split}
        ||\bar{x}(T_{n+1})|| &= \frac{||\hat{x}(T_{n+1}^{-})||}{||\hat{x}(T_n)||} ||\bar{x}(T_n)||\\
        &< \frac{\frac{1}{2}}{\frac{1}{1+1/C_a}}||\bar{x}(T_n)||
    \end{split}
\end{equation*}
Choosing $C_{a} > \max\left(c_{\frac{1}{4}},\frac{2}{r_{\frac{1}{4}}}\right) > 2$, proves the claim.
\end{enumerate}
(ii) and (iii) follow along the lines of arguments in \cite{chandru-SB} Lemma 6 (ii) and (iii) respectively.
\end{proof}
Next we consider the middle timescale of $\{b(n)\}$ and re-define the following terms:
\begin{enumerate}
    \item[\bf{(M1)}] Define
    \[t(n) = \sum_{i=0}^{n-1}b(i),n\geq1 \mbox{ with } t(0) = 0.\]
    Let $\psi_n = (x_n,y_n,z_n)$ and 
    \[\bar{\psi}(t) = \psi_{n} + \left(\psi_{n+1} - \psi_{n}\right)\frac{t-t(n)}{t(n+1) - t(n)}, \mbox{ \quad }t \in [t(n), t(n+1)].\]
    
    \item[\bf{(M2)}]Given $t(n), n\geq 0$ and a constant $T>0$ define
    \[T_0 = 0,\]
    \[T_n = \min(t(m):t(m)\geq T_{n-1}+T) n\geq1\]
    One can find a subsequence $\{m(n)\}$ such that $T_n = t(m(n))$ $\forall n$, and $m(n) \rightarrow \infty$ as $n\rightarrow \infty$.
    
    \item[\bf{(M3)}] The scaling sequence is defined as:
    \[r(n) = \max(r(n-1),||\bar{\psi}(T_n)||,1), n\geq1\]
    
    \item[\bf{(M4)}] The scaled iterates for $m(n)\leq k\leq m(n+1)-1$ are:
    \[\hat{x}_{m(n)} = \frac{x_{m(n)}}{r(n)},\mbox{\quad} \hat{y}_{m(n)} = \frac{y_{m(n)}}{r(n)},\mbox{\quad} \hat{z}_{m(n)} = \frac{z_{m(n)}}{r(n)},\]
    \[\hat{x}_{k+1} = \hat{x}_{k} + a(k)\left(\frac{h(c\hat{x}_{k},c\hat{y}_{k},c\hat{z}_{k})}{c} + \hat{M}_{k+1}^{(1)}\right),\]
    \[\hat{y}_{k+1} = \hat{y}_{k} + b(k)\left(\frac{g(c\hat{x}_{k},c\hat{y}_{k},c\hat{z}_{k})}{c} + \hat{M}_{k+1}^{(2)}\right),\]
    \[\hat{z}_{k+1} = \hat{z}_{k} + b(k)\left(\epsilon_{k}^{(3),b} + \hat{M}_{k+1}^{(3)}\right),\]
    where, $c = r(n)$,
    \[\epsilon_{k}^{(3),b} = \frac{c(k)}{b(k)}\frac{f(c\hat{x}_k,c\hat{y}_k,c\hat{z}_k)}{c},\]
    \[\hat{M}_{k+1}^{(1)} =  \frac{M_{k+1}^{(1)}}{r(n)},\]
    \[\hat{M}_{k+1}^{(2)} =  \frac{M_{k+1}^{(2)}}{r(n)},\]
    \[\hat{M}_{k+1}^{(3)} = \frac{c(k)}{a(k)}\frac{M_{k+1}^{(3)}}{r(n)}.\]
    
    \item[\bf{(M5)}] Next, we define the linearly interpolated trajectory for the scaled iterates as follows:
    \[\hat{\psi}(t) = \hat{\psi}_n + (\hat{\psi}_{n+1} - \hat{\psi}_n)\frac{t - t(n)}{t(n+1) - t(n)}, \mbox{\quad} t \in [t(n), t(n+1)].\]
    
    \item[\bf{(M6)}] Let \(\psi_{n}(t) = (x_n(t),y_n(t),z_n(t)), \mbox{\quad} t\in [T_n,T_{n+1}]\) denote the trajectory of the ODE:
    \[\dot{x}(t) = h_{r(n)}(x(t),y(t),z(t)),\]
    \[\dot{y}(t) = g_{r(n)}(y(t),z(t)),\]
    \[\dot{z}(t) = 0,\]
    with $x_n(T_n) = \hat{x}(T_n)$, $y_n(T_n) = \hat{y}(T_n)$ and $z_n(T_n) = \hat{z}(T_n)$.
\end{enumerate}  
As before we state a few lemmas for ODEs with one external input. These follow along the lines of Lemmas 2-5 of \cite{chandru-SB}.
    Let $y_{c}^{z(t)}(t,y)$ and $y_{\infty}^{z(t)}(t,y)$ denote the solution to the ODEs \[\dot{y}(t) = g_c(y(t),z(t)), \mbox{\quad}t\geq 0,\]
    \[\dot{y}(t) = g_{\infty}(y(t),z(t)), \mbox{\quad}t\geq 0,\] respectively,
    with initial condition $y\in \mathbb{R}^{d_1}$ and the external input $z(t) \in \mathbb{R}^{d_3}$.
    \begin{lemma}
    \label{middle_Aux_a_1}
        Let $K \subset \mathbb{R}^{d_1}$ be a compact set and $z\in \mathbb{R}^{d_3}$. Then under $\bf{(B6)}$, given $\delta>0$, $\exists T_{\delta}>0$ such that $\forall y \in K$
        \[y^{z}_{\infty}(t,y) \in B(\Gamma_{\infty}(z),\delta), \forall \delta\geq T_{\delta}.\]
    \end{lemma}
    \begin{lemma}
    \label{middle_Aux_a_2}
        Let $y\in \mathbb{R}^{d_2}$, $z\in \mathbb{R}^{d_3}$, $[0,T]$ be a given time interval and $r>0$. Let $z'(t) \in B(z,r), \forall t \in [0,T]$, then 
        \[||y_c^{z'(t)}(t,y) - y_{\infty}^{z}(t,y)|| \leq (\epsilon(c) + Lr)Te^{LT}, \mbox{\quad}\forall t \in [0,T],\]
        where $\epsilon(c) \rightarrow 0$ as $c \rightarrow \infty$.
    \end{lemma}
    \begin{lemma}
    \label{middle_Aux_a_3}
        Let $z\in \mathbb{R}^{d_3}$ then given $\epsilon >0$ and $T>0$, $\exists c_{\epsilon,T}>0$, $\delta_{\epsilon,T}>0$ and $r_{\epsilon,T}>0$ such that $\forall t \in [0,T)$, $\forall y \in B(\Gamma_{\infty}(z),\delta_{\epsilon,T})$ $\forall c > c_{\epsilon,T}$ and external input $z'(s) \in B(z,r_{\epsilon,T})$,
        \[y_c^{z'(t)}(t,y) \in B(\Gamma_{\infty}(z),2\epsilon) \mbox{\quad} \forall t \in [0,T].\]
    \end{lemma}
    \begin{lemma}
    \label{middle_Aux_a_4}
        Let $y \in B(0,1) \subset \mathbb{R}^{d_2}, z\in K'\subset\mathbb{R}^{d_3},$ and $\bf{(B7)}$ holds. Then given $\epsilon>0, \exists c_{\epsilon}\geq 1, r_{\epsilon}>0$ and $T_{\epsilon}>0$ such that for any external input satisfying \(z'(s) \in B(z,r_{\epsilon})\), $\forall s \in [0,T]$,
        \[||y_c^{z'(t)}(t,y) - \Gamma_{\infty}(z)|| \leq 2\epsilon, \mbox{\quad} \forall c > c_{\epsilon}, t\geq T_{\epsilon}.\]
    \end{lemma}
\begin{lemma}
    \label{middle1}
    Under $\bf{(B1)}-\bf{(B3)}$,
    \begin{enumerate}
        \item[(i)] For \(0\leq k \leq m(n+1)-m(n)\),
        \(||\hat{\psi}(t(m(n)+k))|| \leq K^{(2)}\)  a.s. for some constant $K^{(2)}>0.$
        
        \item[(ii)] For sufficiently large $n$, we have  \(\sup_{[T_n,T_{n+1}})||\hat{y}(t) - y_n(t)|| = \epsilon(c) LTe^{L(L+1)T} \mbox{ a.s. where } \epsilon(c)\rightarrow0 \mbox{ as } c\rightarrow\infty\)
    \end{enumerate}
\end{lemma}
\begin{proof}
    See Lemma 9 of \cite{chandru-SB}
\end{proof}
\begin{theorem}
    \label{Middle_theorem}
    Assume $\bf{(B1)}$-${\bf{(B4)}}$ and $\bf{(B6)-(B8)}$ hold. Then, with $C^*_a$ as defined in Theorem \ref{fast_thm},
    \begin{enumerate}
        \item[(i)] For large $n$ and $T=T_{1/8(C^*_a+1)}$ (here $T$ is the sampling frequency as in \textbf{(M2)} and $T_{1/8(C^*_a+1)}$ is $T_{\epsilon}$ as in Lemma \ref{middle_Aux_a_4} with $\epsilon=1/8(C^*_a+1)$), if $||\bar{y}(T_n)|| > C_b(1 + ||\bar{z}(T_n)||)$, for some $C_b>0$, then $||\bar{y}(T_{n+1})|| < \frac{5}{8}||\bar{y}(T_n)||$.
        \item[(ii)] $||\bar{y}(T_n)|| \leq C_b^*\left(1+||\bar{z}(T_n)||\right)$, for some $C_b^*>0$
        \item[(iii)] $||y_n|| \leq K_b^*(1 + ||z_n||)$, for some $K_b^*>0$
    \end{enumerate}
\end{theorem}
\begin{proof}
\begin{enumerate}
    \item[(i)] Since \(||\bar{y}(T_n)||>C_b(1 + ||\bar{z}(T_n)||)\), \(r(n) > C_b\).
We first show that $||\hat{z}(T_n)|| < \frac{1}{C_b}$. 
\[||\hat{z}(T_n)|| = \frac{||\bar{z}(T_n)||_p}{r(n)}\leq\frac{||\bar{z}(T_n)||_p}{(||\bar{y}(T_n)||_p^p + ||\bar{z}(T_n)||_p^p)^\frac{1}{p}}\]
Since $||\bar{y}(T_n)||_p>C_b(1 + ||\bar{z}(T_n)||)$, $||\bar{y}(T_n)||_p^p ) > C_b^p||\bar{z}(T_n)||_p^p$. Therefore,
\begin{equation*}
\begin{split}
     ||\hat{z}(T_n)|| &< \frac{||\bar{z}(T_n)||_p}{\left((1+C_b^p)||\bar{z}(T_n)||_p^p\right)^\frac{1}{p}}\\
     &=\frac{1}{(1+C_b^p)^{\frac{1}{p}}}\\
     &< \frac{1}{C_b}
\end{split}
\end{equation*}
Next we show that $||\hat{y}(T_n)||>\frac{1}{(C^*_a+1)(2+\frac{1}{C_b})}$, where $C^*_a$ is as defined in Theorem \ref{fast_thm}.
Here again we are considering the case when the iterates are blowing up. Therefore let $r(n) = ||\bar{\psi}(T_n)||$. Now, from Theorem \ref{fast_thm} ,we know $||\bar{x}(T_n)||\leq K_a^*(1 + ||\bar{y}(T_n)|| + ||\bar{z}(T_n)||)$ and therefore, $r(n) \leq K_a^*(1 + ||\bar{y}(T_n)|| + ||\bar{z}(T_n)||) + ||\bar{y}(T_n)|| + ||\bar{z}(T_n)||$. With this we have,
\begin{equation*}
    \begin{split}
        ||\hat{y}(T_n)||_p &\geq \frac{||\bar{y}(T_n)||}{C^*_a + (C^*_a + 1)(||\bar{y}(T_n)||_p^p + ||\bar{z}(T_n)||_p^p)^\frac{1}{p}}\\
        & > \frac{1}{C^*_a + (C^*_a + 1)(1 + \frac{1}{C_b})}\\
        &>\frac{1}{(C^*_a + 1)(2 + \frac{1}{C_b})}.
    \end{split}
\end{equation*}
Now we proceed as in Theorem \ref{fast_thm} (i). Let $z'(t-T_n) = z_n(t)$ $\forall t \in [T_n, T_{n+1}]$. From Lemma \ref{middle_Aux_a_4}, $\exists r_{1/8(C^*_a+1)}, c_{1/8(C^*_a+1)}, T_{1/8(C^*_a+1)}>0$ such that 
\[||y_{c}^{z'(t)}(t,\hat{x}(T_n))||\leq \frac{1}{8(C^*_a+1)}, \forall t \geq T_{1/8(C^*_a+1)}, \forall c \geq c_{1/8(C^*_a+1)},\]
whenever $z'(t) \in B(0,r_{1/8(C^*_a+1)})$. Choose $T = T_{1/8(C^*_a+1)}$. Since $\dot{z}(t) = 0$ for the ODE defined in $\textbf{(M6)}$ and $z'(t-T_n) = z_n(t) = \hat{z}(T_n)$ $\forall t \in [T_{n},T_{n+1}]$ and we choose $C_{b} > \max\left(c_{1/8(C^*_a+1)},\frac{2}{r_{1/8(C^*_a+1)}}\right)$ from $||\hat{z}(T_n)|| < \frac{1}{C_b}$, it follows that $z'(s) \in B(0,r_{1/8(C^*_a+1)})$ $\forall s \in [0,T].$ Using Lemma \ref{middle1}(ii), $\exists C_1>0$ s.t. $||\hat{y}(T_{n+1}^{-}) - y_{n}(T_{n+1})|| < \frac{1}{8(C^*_a+1)}$ for large enough $n$ and $r(n)>C_1$. Choose $C_b > \max(c_{1/8(C^*_a+1)},\frac{2}{r_{1/8(C^*_a+1)}}, C_1)$. Also observe that $||y_{n}(T_{n+1})|| = ||y^{z'(t)}_{r(n)}(T_{n+1} - T_{n}, \hat{y}(T_{n}))|| \leq \frac{1}{8(C^*_a+1)}$. Using these, we have $||\hat{y}(T_{n+1}^{-})|| \leq ||\hat{y}(T_{n+1}^{-})-y_{n}(T_{n+1})|| + ||y_{n}(T_{n+1})|| \leq \frac{1}{4(C^*_a+1)}$. Finally since \[\frac{||\bar{y}(T_{n+1})||}{||\bar{y}(T_{n})||} = \frac{||\hat{y}(T_{n+1}^{-})||}{||\hat{y}(T_n)||},\] we have 
\begin{equation*}
    \begin{split}
        ||\bar{y}(T_{n+1})|| &= \frac{||\hat{y}(T_{n+1}^{-})||}{||\hat{y}(T_n)||} ||\bar{y}(T_n)||\\
        &< \frac{\frac{1}{4(C^*_a+1)}}{\frac{1}{(C^*_a+1)(2+1/C_b)}}||\bar{x}(T_n)||\\
        &<\frac{2+\frac{1}{C_b}}{4}
    \end{split}
\end{equation*}
Choosing $C_{b} > \max\left(c_{1/8(C^*_a+1)},\frac{2}{r_{1/8(C^*_a+1)}}, C_1\right) > 2$, proves the claim.

\noindent (ii) and (iii) follow along the lines of arguments in \cite{chandru-SB}, Lemma 6 (ii) and (iii), respectively.
\end{enumerate}
\end{proof}
\noindent Finally we consider the slowest timescale corresponding to $\{c(n)\}$. As before we redefine the terms as follows:
\begin{enumerate}
    \item[\bf{(S1)}] Define
    \[t(n) = \sum_{i=0}^{n-1}c(i), n\geq 0 \mbox{ with } t(0) = 0\]
    Let $\psi_n = (x_n,y_n,z_n)$ and 
    \[\bar{\psi}(t) = \psi_{n} + \left(\psi_{n+1} - \psi_{n}\right)\frac{t-t(n)}{t(n+1) - t(n)}, \mbox{ \quad }t \in [t(n), t(n+1)].\]
    
    \item[\bf{(S2)}]Given $t(n), n\geq 0$ and a constant $T>0$ define
    \[T_0 = 0,\]
    \[T_n = \min(t(m):t(m)\geq T_{n+1}+T), n\geq 1\]
    There exists some subsequence $\{m(n)\}$ such that $T_n = t(m(n))$ and $m(n) \rightarrow \infty$ as $n\rightarrow \infty$.
    
    \item[\bf{(S3)}] The scaling sequence is defined as:
    \[r(n) = \max(r(n-1),||\bar{\psi}(T_n)||,1), n\geq 1\]
    
    \item[\bf{(S4)}] The scaled iterates for $m(n)\leq k\leq m(n+1)-1$ are:
    \[\hat{x}_{m(n)} = \frac{x_{m(n)}}{r(n)},\mbox{\quad} \hat{y}_{m(n)} = \frac{y_{m(n)}}{r(n)},\mbox{\quad} \hat{z}_{m(n)} = \frac{z_{m(n)}}{r(n)},\]
    \[\hat{x}_{k+1} = \hat{x}_{k} + a(k)\left(\frac{h(c\hat{x}_{k},c\hat{y}_{k},c\hat{z}_{k})}{c} + \hat{M}_{k+1}^{(1)}\right),\]
    \[\hat{y}_{k+1} = \hat{y}_{k} + b(k)\left(\frac{g(c\hat{x}_{k},c\hat{y}_{k},c\hat{z}_{k})}{c} + \hat{M}_{k+1}^{(2)}\right),\]
    \[\hat{z}_{k+1} = \hat{z}_{k} + c(k)\left(\frac{f(c\hat{x}_{k},c\hat{y}_{k},c\hat{z}_{k})}{c} + \hat{M}_{k+1}^{(3)}\right),\]
    where, $c = r(n)$,
    \[\hat{M}_{k+1}^{(1)} =  \frac{M_{k+1}^{(1)}}{r(n)}, \hat{M}_{k+1}^{(2)} =  \frac{M_{k+1}^{(2)}}{r(n)}, \hat{M}_{k+1}^{(3)} =\frac{M_{k+1}^{(3)}}{r(n)}.\]
    
    \item[\bf{(S5)}] Next we define the linearly interpolated trajectory for the scaled iterates as follows:
    \[\hat{\psi}(t) = \hat{\psi}_n + (\hat{\psi}_{n+1} - \hat{\psi}_n)\frac{t - t(n)}{t(n+1) - t(n)}, \mbox{\quad} t \in [t(n), t(n+1)].\]
    
    \item[\bf{(S6)}] Let \(\psi_{n}(t) = (x_n(t),y_n(t),z_n(t)), \mbox{\quad} t\in [T_n,T_{n+1}]\) denote the trajectory of the ODE:
    \[\dot{x}(t) = h_{r(n)}(x(t),y(t),z(t)),\]
    \[\dot{y}(t) = g_{r(n)}(y(t),z(t)),\]
    \[\dot{z}(t) = f_{r(n)}(z(t)),\]
    with $x_n(T_n) = \hat{x}(T_n)$, $y_n(T_n) = \hat{y}(T_n)$ and $z_n(T_n) = \hat{z}(T_n)$.
\end{enumerate}
We again state some results on ODEs, this time with no external input. These again follow along the lines of Lemma 2-5 in \cite{chandru-SB}.
Let $z_{c}(t,z)$ and $z_{\infty}(t,z)$ denote the solution to the ODEs \[\dot{z}(t) = f_c(z(t)), \mbox{\quad}t\geq 0,\]
    \[\dot{z}(t) = f_{\infty}(z(t)), \mbox{\quad}t\geq 0,\] respectively
    with initial condition $z\in \mathbb{R}^{d_3}$.
    \begin{lemma}
    \label{slow_Aux_a_1}
        Let $K \subset \mathbb{R}^{d_3}$ be a compact set . Then under $\bf{(B8)}$, given $\delta>0$, $\exists T_{\delta}>0$ such that $\forall z \in K$
        \[z_{\infty}(t,z) \in B(0,\delta), \forall \delta\geq T_{\delta}.\]
    \end{lemma}
    \begin{lemma}
    \label{slow_Aux_a_2}
        Let $z\in \mathbb{R}^{d_3}$, $[0,T]$ be a given time interval and $r>0$. Then 
        \[||z_c(t,z) - z_{\infty}(t,z)|| \leq (\epsilon(c))Te^{LT}, \mbox{\quad}\forall t \in [0,T],\]
        where $\epsilon(c) \rightarrow 0$ as $c \rightarrow \infty$.
    \end{lemma}
    \begin{lemma}
    \label{slow_Aux_a_3}
        Given $\epsilon >0$ and $T>0$ $\exists c_{\epsilon,T}>0$, $\delta_{\epsilon,T}>0$ and $r_{\epsilon,T}>0$ such that $\forall t \in [0,T)$, $\forall z \in B(0,\delta_{\epsilon,T})$, $\forall c > c_{\epsilon,T}$,
        \[z_c(t,z) \in B(0,2\epsilon) \mbox{\quad}, \forall t \in [0,T].\]
    \end{lemma}
    \begin{lemma}
    \label{slow_Aux_a_4}
        Let $z \in B(0,1) \subset \mathbb{R}^{d_3}$ and let $\bf{(B8)}$ hold. Then given $\epsilon>0, \exists c_{\epsilon}\geq 1, r_{\epsilon}>0$ and $T_{\epsilon}>0$, then 
        \[||z_c(t,z)|| \leq 2\epsilon, \mbox{\quad} \forall c > c_{\epsilon}.\]
    \end{lemma}
    \begin{lemma}
    \label{slow1}
    Under $\bf{(B1)}-\bf{(B3)}$,
    \begin{enumerate}
        \item[(i)] For \(0\leq k \leq m(n+1)-m(n)\),
        \(||\hat{\psi}(t(m(n)+k))|| \leq K^{(3)}\)  a.s. for some constant $K^{(3)}>0.$
        
        \item[(ii)] For sufficiently large $n$, we have  \(\sup_{[T_n,T_{n+1})}||\hat{z}(t) - z_n(t)|| = (\epsilon_1(c) + \epsilon_2(c)) LTe^{L(L+1)T} \mbox{ a.s. where } \epsilon(c)\rightarrow0 \mbox{ as } c\rightarrow\infty.\)
    \end{enumerate}
\end{lemma}
\begin{proof}
    See Lemma 9 (ii) and (iii) of \cite{chandru-SB}.
\end{proof}
\begin{theorem}
\label{slow_thm}
    Under assumptions $\bf{(B1)}$-$\bf{(B4)}$ and $\bf{(B6)}-\bf{(B8)}$ , we have:
    \begin{itemize}
        \item[(i)] Let $C^*_a$ and $C^*_b$ be as in Theorems \ref{fast_thm} and \ref{Middle_theorem} respectively. Then, $||\hat{z}(T_n)||\geq\frac{1}{4 + C^*_aC^*_b + C^*_b}$ for sufficiently large $||\bar{z}(T_n)||$.
        \item[(ii)] For $n$ large, $T = T_{\frac{1}{4}}$(here $T$ is the sampling frequency as in \textbf{(F2)} and $T_{\frac{1}{4}}$ is $T_{\epsilon}$ as in Lemma \ref{Aux_a_4} with $\epsilon=\frac{1}{4}$), if $||\bar{z}(T_n)||>C$, for some $C>0$ then $||\bar{z}(T_{n+1})||<\frac{1}{2}||\bar{z}(T_n)||$
        \item[(iii)] $||\bar{z}(T_n)|| \leq K_c^*$ for some $K_c^*>0$.
        \item[(iv)] $\sup_{n}||z_n||<\infty$ a.s.
    \end{itemize}
\end{theorem}
\begin{proof}
\begin{enumerate}
    \item[(i)] From Theorems \ref{fast_thm} and \ref{Middle_theorem} we know that \(||r(n)||<C^*_a( 1 + ||\bar{y}(T_n)|| + ||\bar{z}(T_n)||) + C^*_b( 1 + ||\bar{z}(T_n)||) + ||\bar{z}(T_n)||\). Therefore,
    \begin{equation*}
        \begin{split}
            ||\hat{z}(T_n)|| &= \frac{||\bar{z}(T_n)||}{r(n)}
            > \frac{||\bar{z}(T_n)||}{C^*_a( 1 + ||\bar{y}(T_n)|| + ||\bar{z}(T_n)||) + C^*_b( 1 + ||\bar{z}(T_n)||) + ||\bar{z}(T_n)||}\\
            & > \frac{||\bar{z}(T_n)||}{C^*_a( 1 + C^*_b( 1 + ||\bar{z}(T_n)||) + ||\bar{z}(T_n)||) + C^*_b( 1 + ||\bar{z}(T_n)||) + ||\bar{z}(T_n)||}\\
            &>\frac{1}{\frac{C^*_a}{||\bar{z}(T_n)||} + \frac{C^*_aC^*_b}{||\bar{z}(T_n)||}+ C^*_aC^*_b + C^*_a + \frac{C^*_b}{||\bar{z}(T_n)||} + C^*_b + 1}\\
            & > \frac{1}{4 + C^*_aC^*_b + C^*_b}, \mbox{\qquad for \quad} ||\bar{z}(T_n)|| > \max{(C^*_a, C^*_b, C^*_aC^*_b)}
        \end{split}
    \end{equation*}
    \item[(ii)] Since, $0\in\mathbb{R}^{d_3}$ is the unique globally asymptotically stable equilibrium, therefore using Lemma \ref{slow_Aux_a_4}, $\exists c_{\frac{1}{4}},T_{\frac{1}{4}}>0$, such that $||z_{c}(t,z)||<\frac{1}{4(4 + C^*_aC^*_b + C^*_b)},$ $\forall c \geq c_{\frac{1}{4}}, t \geq T_{\frac{1}{4}}$. Also, for $||\bar{z}(T_{n})|| > \max{(C^*_a, C^*_b, C^*_aC^*_b)}$ we have $||\hat{z}(T_{n})|| > \frac{1}{4 + C^*_aC^*_b + C^*_b}$ and for sufficiently large $n$, from Lemma \ref{slow1}(ii), $\exists C_2>0$ such that 
    $||\hat{z}(T_{n+1}^{-})-z_{n}(T_{n+1})|| < \frac{1}{4(4 + C^*_aC^*_b + C^*_b)}$ for $r(n)>C_2$. We pick $C = \max(c_{1/4},C_1,\max{(C^*_a, C^*_b, C^*_aC^*_b)})$ and $T= T_{1/4}$. For $n$ large it then follows that 
    $||\hat{z}(T_{n+1}^{-})|| \leq ||\hat{z}(T_{n+1}^{-})-z_{n}(T_{n+1})|| + ||z_{n}(T_{n+1})|| \leq \frac{1}{2(4 + K_a^*C^*_b + C^*_b)}$. Finally, since \[\frac{||\bar{z}(T_{n+1})||}{||\bar{z}(T_{n})||} = \frac{||\hat{z}(T_{n+1}^{-})||}{||\hat{z}(T_n)||},\] it follows that
    \[||\bar{z}(T_{n+1})|| < \frac{1}{2}||\bar{z}(T_{n})||.\]
\end{enumerate}
(iii) and (iv) follow along the lines of arguments as in Lemma 10 (iii) and (iv) of \cite{chandru-SB}.
\end{proof}
Now from Theorem \ref{slow_thm} (iii), it follows that the slow timescale iterates $z_{n}$ are bounded a.s. ( $||z_n||<\infty$ a.s. ) which in turn implies that the middle timescale iterates $y_n$ are bounded using Theorem \ref{Middle_theorem} ( i.e., $||y_n|| < \infty$ a.s. ). Finally the fast timescale iterates $x_n$ are bounded because of Theorem \ref{fast_thm} and the fact that both middle timescale and slow timescale iterates are bounded showing $||x_n|| <\infty$ a.s. Combining these we have $\sup_{n} (||x_n|| + ||y_n|| + ||z_n||) < \infty$ a.s, thereby proving Theorem \ref{thm_app_stability}.
\end{proof}
The slightly more general version where each iterate could have small perturbation terms as given below:
\begin{equation}
    x_{n+1} = x_{n} + a(n)\left(h(x_{n},y_{n},z_{n}) + M_{n+1}^{(1)} + \varepsilon^{(1)}_n\right),
\end{equation}
\begin{equation}
    y_{n+1} = y_{n} + b(n)\left(g(x_{n},y_{n},z_{n}) + M_{n+1}^{(2)} + \varepsilon^{(2)}_n\right),
\end{equation}
\begin{equation}
    z_{n+1} = z_{n} + c(n)\left(f(x_{n},y_{n},z_{n}) + M_{n+1}^{(3)} + \varepsilon^{(3)}_n\right),
\end{equation}
with $\epsilon_{n}^{(k)} = o(1), k=1,2,3$ can be shown to converge to the same solution. Since the additional error terms are $o(1)$, their contribution is asymptotically negligible. See arguments in third extension of (Chapter 2, pp. 17 of \citet{Borkar_Book} ) that handles this case for one-timescale iterates. \section*{A3 \quad Convergence of GTD-2 M and TDC-M}
Here we provide the asymptotic convergence guarantees of the momentum variants of the remaining two Gradient TD methods namely \textbf{GTD2-M} and \textbf{TDC-M}. The analysis is similar to that of \textbf{GTD-M} in Theorem \ref{theorem_GTD-M_3TS} and is provided here for completeness. We show that the assumptions \textbf{(B1) - (B7)} of the main paper are satisfied and thereby invoke Theorem \ref{theorem_3TS} to show convergence.
\subsection*{A3.1 \quad Asymptotic convergence of GTD2-M}
We re-write the iterates for GTD2-M below:
\begin{equation}
\label{gtd2_M_1_app}
    \theta_{t+1} = \theta_{t} + \alpha_{t}(\phi_{t} - \gamma\phi_{t}')\phi_{t}^{T}u_{t} + \eta_{t}(\theta_{t} - \theta_{t-1}),
\end{equation}
\begin{equation}
\label{gtd2_M_2_app}
     u_{t+1} = u_{t} + \beta_{t}(\delta_{t} - \phi_{t}^{T}u_{t})\phi_{t}.
\end{equation}
As before, choosing \(\eta_t =\frac{\varrho_t-w\alpha_t}{\varrho_{t-1}}\), where $\{\varrho_t\}$ is a positive sequence and $w\in\mathbb{R}$ is a constant, we can decompose the two iterates into three recursions as below:
\begin{gather}
    \label{GTD2-M-1_app}
    v_{t+1} = v_{t} + \xi_{t}\left((\phi_{t} - \gamma\phi_{t}')\phi_{t}^{T}u_{t} - w v_{t}\right)\\
    \label{GTD2-M-2_app}
    u_{t+1} = u_{t} + \beta_{t} (\delta_{t}\phi_{t} - \phi_t\phi_t^Tu_{t})\\
    \label{GTD2-M-3_app}
    \theta_{t+1} = \theta_{t} + \varrho_{t}(v_{t} + \varepsilon_{t})
\end{gather}
\begin{theorem}
    \label{theorem_GTD2-M_3TS_app}
     Assume $\mathbfcal{A}$\textbf{\textup{\ref{3A1}}}, $\mathbfcal{A}$\textbf{\textup{\ref{3A3}}} and $\mathbfcal{A}$\textbf{\textup{\ref{3A4}}} hold and let $w>0$. Then, the GTD2-M iterates given by \eqref{gtd2_M_1_app} and \eqref{gtd2_M_2_app} satisfy \(\theta_{n} \rightarrow \theta^{*} = -\bar{A}^{-1}\bar{b}\) a.s. as \(n\rightarrow \infty\). 
\end{theorem}
\begin{proof}
We transform the iterates given by \eqref{GTD2-M-1_app}, \eqref{GTD2-M-2_app} and \eqref{GTD2-M-3_app} into the standard SA form given by \eqref{main_iter_x}, \eqref{main_iter_y} and \eqref{main_iter_z}. Let $\mathcal{F}_t = \sigma(u_0, v_0, \theta_0, r_{j+1},\phi_j, \phi_j': j < t)$. Let, $A_t = \phi_t(\gamma\phi_t'-\phi_t)^T$ and $b_t = r_{t+1}\phi_t$. Then, \eqref{GTD2-M-1_app} can be re-written as:
\begin{equation*}
    v_{t+1} = v_t + \xi_{t}\left(h(v_t,u_t,\theta_t) + M_{t+1}^{(1)}\right)
\end{equation*}
where,
\begin{equation*}
    \begin{split}
        h(v_t,u_t,\theta_t) &= \mathbb{E}[(\phi_t - \gamma\phi_t')\phi_t^Tu_t - w v_t|\mathcal{F}_t]\\ 
        &= -\bar{A}^Tu_t-wv_t .\\
        M_{t+1}^{(1)} = -A_t^Tu_t& - w v_t - h(v_t,u_t,\theta_t)  = (\bar{A}^T-A_t^T)u_t .
    \end{split}
\end{equation*}
Next, \eqref{GTD2-M-2_app} can be re-written as:
\begin{equation*}
    \begin{split}
        u_{t+1} &= u_t + \beta_t\left(g(v_t,u_t,\theta_t) + M_{t+1}^{(2)}\right)\\
\end{split}
\end{equation*}
where,
\begin{equation*}
    \begin{split}
        g(v_t,u_t,\theta_t) &= \mathbb{E}[\delta_t\phi_t - \phi_t\phi_t^T u_t|\mathcal{F}_t]
        = \bar{A}\theta_t + \bar{b} - \bar{C}u_t\\
        M_{t+1}^{(2)} &= A_t\theta_t + b_t - C_t u_t - g(v_t,u_t,\theta_t)\\ 
        &= (A_t-\bar{A})\theta_t + (b_t-\bar{b}) + (\bar{C} - C_t)u_t.
    \end{split}
\end{equation*}
Here, $C_t = \phi_t\phi_t^T$ and $\bar{C} = \mathbb{E}[\phi_t\phi_t^T]$. Finally, \eqref{GTD2-M-3_app} can be re-written as:
\begin{equation*}
    \begin{split}
        \theta_{t+1} = \theta_t + \varrho_{t}\left(f(v_t,u_t,\theta_t) + \varepsilon_t + M_{t+1}^{(3)}\right)
    \end{split}
\end{equation*}
where,
\[
        f(v_t,u_t,\theta_t) = v_t \mbox{ and }
        M_{t+1}^{(3)} = 0.
\]
The functions $h,g,f$ are linear in $v,u,\theta$ and hence Lipchitz continuous, therefore satisfying $\bf{(B1)}$. We choose the step-size sequences such that they satisfy $\bf{(B2)}$. One popular choice is \[\xi_t = \frac{1}{(t+1)^{\xi}}, \beta_t = \frac{1}{(t+1)^{\beta}}, \varrho_t= \frac{1}{(t+1)^{\varrho}},\frac{1}{2}<\xi<\beta<\varrho\leq1.\]
Next, $M_{t+1}^{(1)},M_{t+1}^{(2)}$ and $M_{t+1}^{(3)}$ $t\geq0$, are martingale difference sequences w.r.t $\mathcal{F}_t$ by construction. Next, \[\mathbb{E}[||M_{t+1}^{(1)}||^2|\mathcal{F}_t] \leq ||(\bar{A}^T - A_t^T)||^2 ||u_t||^2,\] \[\mathbb{E}[||M_{t+1}^{(2)}||^2|\mathcal{F}_t] \leq 3(||(A_t-\bar{A})||^2 ||\theta_t||^2 + ||(b_t-\bar{b})||^2 + ||(\bar{C}-C_t)||^2||u_t||^2).\] The first part of $\bf{(B3)}$ is satisfied with $K_1 = ||(\bar{A}^T-A_{t}^T)||^2$, $K_2 = 3\max(||A_t - \bar{A}||^2,||b_t-\bar{b}||^2,||(\bar{C}-C_t)||^2)$ and any $K_3>0$.
The fact that $K_1,K_2<\infty$ follows from the bounded features and bounded rewards assumption in $\mathbfcal{A}$\textbf{\textup{\ref{3A1}}}. Next, observe that $||\varepsilon_t^{(3)}||=\xi_t||\left((\phi_t - \gamma\phi_t')\phi_t^Tu_t - w v_t\right)||\rightarrow0$ since $\xi_t\rightarrow0 \mbox{ as } t\rightarrow\infty$. For a fixed $u,\theta \in \mathbb{R}^d$, consider the ODE
\[\dot{v}(t) = -\bar{A}^Tu - w v(t).\]
For $w>0$, $\lambda(u,\theta) = -\frac{\bar{A}^Tu}{w}$ is the unique g.a.s.e, is linear and therefore Lipchitz continuous. This satisfies $\bf{(B4)}$(i). Next, for a fixed $\theta \in \mathbb{R}^d$,
\[\dot{u}(t) = \bar{A}\theta + \bar{b} -\bar{C}u(t),\]
has $\Gamma(\theta) = \bar{C}^{-1}(\bar{A}\theta + \bar{b})$ as its unique g.a.s.e because $-\bar{C}^{-1}$ is negative definite. Also $\Gamma(\theta)$ is linear in $\theta$ and therefore Lipschitz. This satisfies $\bf{(B4)}(ii)$. Finally, to satisfy $\bf{(B4)}(iii)$, consider,
\begin{equation*}
    \begin{split}
        \dot{\theta}(t)
       & = \frac{-\bar{A}^T\bar{C}^{-1}\bar{A}\theta(t)-\bar{A}^T\bar{C}^{-1}\bar{b}}{w}.
    \end{split}
\end{equation*}
Since $\bar{A}$ is negative definite and $\bar{C}$ is positive definite, therefore, $-\bar{A}^T\bar{C}^{-1}\bar{A}$ is negative definite. Therefore, $\theta^* = -\bar{A}^{-1}\bar{b}$ is the unique g.a.s.e. 

Next, we show that the sufficient conditions for stability of the three iterates are satisfied. The function, $h_c(v,u,\theta) = \frac{-c\bar{A}^Tu-wcv}{c} = -\bar{A}^Tu-wv \rightarrow h_{\infty}(v,u,\theta) = -\bar{A}^Tu-wv$ uniformly on compacts as $c\rightarrow\infty$. The limiting ODE: 
\[\dot{v}(t) = -\bar{A}^Tu-wv(t)\]
has $\lambda_{\infty}(u,\theta) = -\frac{\bar{A}^Tu}{w}$ as its unique g.a.s.e. $\lambda_{\infty}$ is Lipschitz with $\lambda_{\infty}(0,0) = 0$, thus satisfying assumption $\bf{(B5)}$.

\noindent The function, $g_c(u,\theta) = \frac{c\bar{A}\theta + \bar{b} - c\bar{C}u}{c} = \bar{A}\theta-\bar{C}u+\frac{\bar{b}}{c} \rightarrow g_{\infty}(u,\theta) = \bar{A}\theta - \bar{C}u$ uniformly on compacts as $c\rightarrow\infty$. The limiting ODE \[\dot{u}(t) = \bar{A}\theta - \bar{C}u(t)\]
has $\Gamma_{\infty}(\theta) = \bar{C}^{-1}\bar{A}\theta$ as its unique g.a.s.e. since $-\bar{C}$ is negative definite. $\Gamma_{\infty}$ is Lipchitz with $\Gamma_{\infty}(0) = 0$. Thus assumption $\bf{(B6)}$ is satisfied.

\noindent Finally, $f_c(\theta) = \frac{-c\bar{A}^T\bar{C}^{-1}\bar{A}\theta}{cw} \rightarrow f_{\infty} = \frac{-\bar{A}^T\bar{A}\theta}{w}$ uniformly on compacts as $c\rightarrow \infty$ and the ODE:
\[\dot{\theta}(t) = -\frac{\bar{A}^T\bar{C}^{-1}\bar{A}\theta(t)}{w}\]
has origin in $\mathbb{R}^d$as its unique g.a.s.e. This ensures the final condition $\bf{(B7)}$. By theorem \ref{theorem_3TS}, 
\[
\begin{pmatrix}
    v_t\\
    u_t\\
    \theta_t
\end{pmatrix}
\rightarrow
\begin{pmatrix}
    \lambda(\Gamma(-\bar{A}^{-1}\bar{b}),-\bar{A}^{-1}\bar{b})\\
    \Gamma(-\bar{A}^{-1}\bar{b})\\
    -\bar{A}^{-1}\bar{b}.
\end{pmatrix}
=
\begin{pmatrix}
    0\\
    0\\
    -\bar{A}^{-1}\bar{b}.
\end{pmatrix}
\]
Specifically, $\theta_t \rightarrow -\bar{A}^{-1}\bar{b}$.
\end{proof}
\subsection*{A3.2 \quad Asymptotic Convergence of TDC-M}
We re-write the iterates for TDC-M below:
\begin{equation}
\begin{split}
\label{tdc_M_1_app}
    \theta_{t+1} = \theta_{t} + \alpha_{t}(\delta_{t}\phi_{t} - \gamma\phi_{t}'(\phi_{t}^{T}u_{t}))+ \eta_{t}(\theta_{t} - \theta_{t-1}),
\end{split}
\end{equation}
\begin{equation}
\label{tdc_M_2_app}
     u_{t+1} = u_{t} + \beta_{t}(\delta_{t} - \phi_{t}^{T}u_{t})\phi_{t}.
\end{equation}
As before, choosing \(\eta_t =\frac{\varrho_t-w\alpha_t}{\varrho_{t-1}}\), where $\{\varrho_t\}$ is a positive sequence and $w\in\mathbb{R}$ is a constant, we can decompose the two iterates into three recursions as below:
\begin{gather}
    \label{TDC-M-1_app}
    v_{t+1} = v_{t} + \xi_{t}\left(\delta_t\phi_t - \gamma\phi_{t}'\phi_{t}^{T}u_{t} - w v_{t}\right)\\
    \label{TDC-M-2_app}
    u_{t+1} = u_{t} + \beta_{t} (\delta_{t}\phi_{t} - \phi_t\phi_t^Tu_{t})\\
    \label{TDC-M-3_app}
    \theta_{t+1} = \theta_{t} + \varrho_{t}(v_{t} + \varepsilon_{t})
\end{gather}
\begin{theorem}
    \label{theorem_TDC-M_3TS_app}
     Assume $\mathbfcal{A}$\textbf{\textup{\ref{3A1}}}, $\mathbfcal{A}$\textbf{\textup{\ref{3A3}}} and $\mathbfcal{A}$\textbf{\textup{\ref{3A4}}} hold and let $w>0$. Then, the TDC-M iterates given by \eqref{tdc_M_1_app} and \eqref{tdc_M_2_app} satisfy \(\theta_{n} \rightarrow \theta^{*} = -\bar{A}^{-1}\bar{b}\) a.s. as \(n\rightarrow \infty\). 
\end{theorem}
\begin{proof}
We transform the iterates given by \eqref{TDC-M-1_app}, \eqref{TDC-M-2_app} and \eqref{TDC-M-3_app} into the standard SA form given by \eqref{main_iter_x}, \eqref{main_iter_y} and \eqref{main_iter_z}. Let $\mathcal{F}_t = \sigma(u_0, v_0, \theta_0, r_{j+1},\phi_j, \phi_j': j < t)$. Let, $A_t = \phi_t(\gamma\phi_t'-\phi_t)^T$ and $b_t = r_{t+1}\phi_t$. Then, \eqref{TDC-M-1_app} can be re-written as:
\begin{equation*}
    v_{t+1} = v_t + \xi_{t}\left(h(v_t,u_t,\theta_t) + M_{t+1}^{(1)}\right)
\end{equation*}
where,
\begin{equation*}
    \begin{split}
        h(v_t,u_t,\theta_t) &= \mathbb{E}[\delta_t\phi_t - \gamma\phi_{t}'\phi_{t}^{T}u_{t} - w v_{t}|\mathcal{F}_t]\\ 
        &=\bar{A}\theta_t + \bar{b} -\gamma\mathbb{E}[\phi_t'\phi_t^T]u_t-wv_t .\\
        M_{t+1}^{(1)} &= \delta_t\phi_t - \gamma\phi_{t}'\phi_{t}^{T}u_{t} - w v_{t} - h(v_t,u_t,\theta_t)\\ & = (A_t-\bar{A})\theta_t + (b_t-\bar{b}) + \gamma(\mathbb{E}[\phi_t'\phi_t^T] - \phi_t'\phi_t^T)u_t.
    \end{split}
\end{equation*}
Next, \eqref{TDC-M-1_app} can be re-written as:
\begin{equation*}
    \begin{split}
        u_{t+1} &= u_t + \beta_t\left(g(v_t,u_t,\theta_t) + M_{t+1}^{(2)}\right)\\
\end{split}
\end{equation*}
where,
\begin{equation*}
    \begin{split}
        g(v_t,u_t,\theta_t) &= \mathbb{E}[\delta_t\phi_t - \phi_t\phi_t^T u_t|\mathcal{F}_t]
        = \bar{A}\theta_t + \bar{b} - \bar{C}u_t\\
        M_{t+1}^{(2)} &= A_t\theta_t + b_t - C_t u_t - g(v_t,u_t,\theta_t)\\ 
        &= (A_t-\bar{A})\theta_t + (b_t-\bar{b}) + (\bar{C} - C_t)u_t.
    \end{split}
\end{equation*}
Here, $C_t = \phi_t\phi_t^T$ and $\bar{C} = \mathbb{E}[\phi_t\phi_t^T]$. Finally, \eqref{TDC-M-1_app} can be re-written as:
\begin{equation*}
    \begin{split}
        \theta_{t+1} = \theta_t + \varrho_{t}\left(f(v_t,u_t,\theta_t) + \varepsilon_t + M_{t+1}^{(3)}\right)
    \end{split}
\end{equation*}
where,
\[
        f(v_t,u_t,\theta_t) = v_t \mbox{ and }
        M_{t+1}^{(3)} = 0.
\]
The functions $h,g,f$ are linear in $v,u,\theta$ and hence Lipchitz continuous, therefore satisfying $\bf{(B1)}$. We choose the step-size sequences such that they satisfy $\bf{(B2)}$. One popular choice is \[\xi_t = \frac{1}{(t+1)^{\xi}}, \beta_t = \frac{1}{(t+1)^{\beta}}, \varrho_t= \frac{1}{(t+1)^{\varrho}}, \frac{1}{2}<\xi<\beta<\varrho\leq1.\]
Observe that, $M_{t+1}^{(1)},M_{t+1}^{(2)}$ and $M_{t+1}^{(3)}$ $t\geq0$, are martingale difference sequences w.r.t $\mathcal{F}_t$ by construction. Next, \[\mathbb{E}[||M_{t+1}^{(1)}||^2|\mathcal{F}_t] \leq 3(||(A_t-\bar{A})||^2||\theta_t||^2 + ||(b_t-\bar{b})||^2 + \gamma(||\mathbb{E}[\phi_t^{'}\phi_t^T]-\phi_t'\phi_t^T||^2)||u_t||^2),\]
\[ \mathbb{E}[||M_{t+1}^{(2)}||^2|\mathcal{F}_t] \leq 3(||(A_t-\bar{A})||^2 ||\theta_t||^2 + ||(b_t-\bar{b})||^2 + ||(\bar{C}-C_t)||^2||u_t||^2)\]. 
The first part of $\bf{(B3)}$ is satisfied with $K_1 =3 \max(||(A_t-\bar{A})||^2, ||(b_t-\bar{b})||^2, \gamma(||\mathbb{E}[\phi_t^{'}\phi_t^T]-\phi_t'\phi_t^T||^2))$, $K_2 = 3\max(||A_t - \bar{A}||^2,||b_t-\bar{b}||^2,||(\bar{C}-C_t)||^2)$ and any $K_3>0$.
The fact that $K_1,K_2<\infty$ follows from the bounded features and bounded rewards assumption in $\mathbfcal{A}$\textbf{\textup{\ref{3A1}}}. Next, observe that $||\varepsilon_t^{(3)}||=\xi_t||\left((\phi_t - \gamma\phi_t')\phi_t^Tu_t - w v_t\right)||\rightarrow0$ since $\xi_t\rightarrow0 \mbox{ as } t\rightarrow\infty$. 
For a fixed $u,\theta \in \mathbb{R}^d$, consider the ODE
\[\dot{v}(t) = \bar{A}\theta + \bar{b} -\gamma\mathbb{E}[\phi_t'\phi_t^T]u-wv(t).\]
For $w>0$, $\lambda(u,\theta) = \frac{\bar{A}\theta+\bar{b}-\gamma\mathbb{E}[\phi_t'\phi_t^T]u}{w}$ is the unique g.a.s.e, is linear and therefore Lipchitz continuous. This satisfies $\bf{(B4)}$(i). Next, for a fixed $\theta \in \mathbb{R}^d$,
\[\dot{u}(t) = \bar{A}\theta + \bar{b} -\bar{C}u(t),\]
has $\Gamma(\theta) = \bar{C}^{-1}(\bar{A}\theta + \bar{b})$ as its unique g.a.s.e because $-\bar{C}^{-1}$ is negative definite. Also $\Gamma(\theta)$ is linear in $\theta$ and therefore Lipschitz. This satisfies $\bf{(B4)}(ii)$. Finally, to satisfy $\bf{(B4)}(iii)$, consider,
\begin{equation*}
    \begin{split}
        \dot{\theta}(t)
       & = \frac{(I-\gamma\mathbb{E}[\phi_t'\phi_t^T]\bar{C}^{-1})(\bar{A}\theta(t) + \bar{b})}{w}.
    \end{split}
\end{equation*}

Now, $(I-\gamma\mathbb{E}[\phi_t'\phi_t^T]\bar{C}^{-1})\bar{A}$ = $(\mathbb{E}[\phi_t\phi_t^T]-\gamma\mathbb{E}[\phi_t'\phi_t^T])\bar{C}^{-1}\bar{A} = \mathbb{E}[(\phi_t - \gamma\phi_t')\phi_t^T]\bar{C}^{-1}\bar{A} = -\bar{A}^T\bar{C}^{-1}\bar{A}$. Since, $\bar{A}$ is negative definite and $\bar{C}$ is positive definite, therefore $ -\bar{A}^T\bar{C}^{-1}\bar{A}$ is negative definite and hence the above ODE has $\theta^* = -\bar{A}^{-1}\bar{b}$ as its unique g.a.s.e. 

\noindent Next, we show that the sufficient conditions for stability of the three iterates are satisfied. The function, $h_c(v,u,\theta) = \frac{c\bar{A}\theta + \bar{b}-c\gamma\mathbb{E}[\phi_t'\phi_t^{T}]u-cwv}{c} = \bar{A}\theta_t-\gamma\mathbb{E}[\phi_t'\phi_t^{T}]u_t-wv_t \rightarrow h_{\infty}(v,u,\theta) = \bar{A}\theta_t-\gamma\mathbb{E}[\phi_t'\phi_t^{T}]u_t-wv_t$ uniformly on compacts as $c\rightarrow\infty$. The limiting ODE: 
\[\dot{v}(t) = \bar{A}\theta_t-\gamma\mathbb{E}[\phi_t'\phi_t^{T}]u_t-wv(t)\]
has $\lambda_{\infty}(u,\theta) = \frac{\bar{A}\theta-\gamma\mathbb{E}[\phi_t'\phi_t^{T}]u}{w}$ as its unique g.a.s.e. $\lambda_{\infty}$ is Lipschitz with $\lambda_{\infty}(0,0) = 0$, thus satisfying assumption $\bf{(B5)}$.

\noindent The function, $g_c(u,\theta) = \frac{c\bar{A}\theta + \bar{b} - c\bar{C}u}{c} = \bar{A}\theta-\bar{C}u+\frac{\bar{b}}{c} \rightarrow g_{\infty}(u,\theta) = -\bar{A}\theta - \bar{C}u$ uniformly on compacts as $c\rightarrow\infty$. The limiting ODE 
\[\dot{u}(t) = \bar{A}\theta - \bar{C}u(t)\]
has $\Gamma_{\infty}(\theta) = \bar{C}^{-1}\bar{A}\theta$ as its unique g.a.s.e. since $-\bar{C}$ is negative definite. $\Gamma_{\infty}$ is Lipschitz with $\Gamma_{\infty}(0) = 0$. Thus assumption $\bf{(B6)}$ is satisfied.

\noindent Finally, $f_c(\theta) = \frac{c\bar{A}\theta-c\gamma\mathbb{E}[\phi_t'\phi_t^T]\bar{C}^{-1}\bar{A}\theta}{cw} \rightarrow f_{\infty} = \frac{(I-\gamma\mathbb{E}[\phi_t'\phi_t^T]\bar{C}^{-1})\bar{A}\theta}{w}$ uniformly on compacts as $c\rightarrow \infty$. Consider the ODE:
\[\dot{\theta}(t) = \frac{(I-\gamma\mathbb{E}[\phi_t'\phi_t^T]\bar{C}^{-1})\bar{A}\theta(t)}{w}.\]

Now, $(I-\gamma\mathbb{E}[\phi_t'\phi_t^T]\bar{C}^{-1})\bar{A}$ = $(\mathbb{E}[\phi_t\phi_t^T]-\gamma\mathbb{E}[\phi_t'\phi_t^T])\bar{C}^{-1}\bar{A} = \mathbb{E}[(\phi_t - \gamma\phi_t')\phi_t^T]\bar{C}^{-1}\bar{A} = -\bar{A}^T\bar{C}^{-1}\bar{A}$. Since, $\bar{A}$ is negative definite and $\bar{C}$ is positive definite, therefore $ -\bar{A}^T\bar{C}^{-1}\bar{A}$ is negative definite and hence the above ODE has origin as its unique g.a.s.e. This ensures the final condition $\bf{(B7)}$. By Theorem \ref{theorem_3TS}, 
\[
\begin{pmatrix}
    v_t\\
    u_t\\
    \theta_t
\end{pmatrix}
\rightarrow
\begin{pmatrix}
    \lambda(\Gamma(-\bar{A}^{-1}\bar{b}),-\bar{A}^{-1}\bar{b})\\
    \Gamma(-\bar{A}^{-1}\bar{b})\\
    -\bar{A}^{-1}\bar{b}.
\end{pmatrix}
=
\begin{pmatrix}
    0\\
    0\\
    -\bar{A}^{-1}\bar{b}.
\end{pmatrix}
\]
Specifically, $\theta_t \rightarrow -\bar{A}^{-1}\bar{b}$.
\end{proof}

\section* {A4 \quad Experiment Details}
Here we briefly describe the MDP settings considered in section \ref{sec_exp}. 
\begin{enumerate}
    \item \textbf{Example-1 (Boyan Chain):\label{Ex1}} It consists of a linear arrangement of 14 states. From each of the first 13 states, one can move to the next state or the next to next state with equal probability. The last state is an absorbing state. The reward at each transition is -3 except the transition from state-6 to state-7 where it is -2. The discount factor $\gamma$ is set to $0.95$. The following figure shows the corresponding MDP for 7 state Boyan Chain.
    \begin{figure}[H]
    \centering
    {
\includegraphics[width=\linewidth]{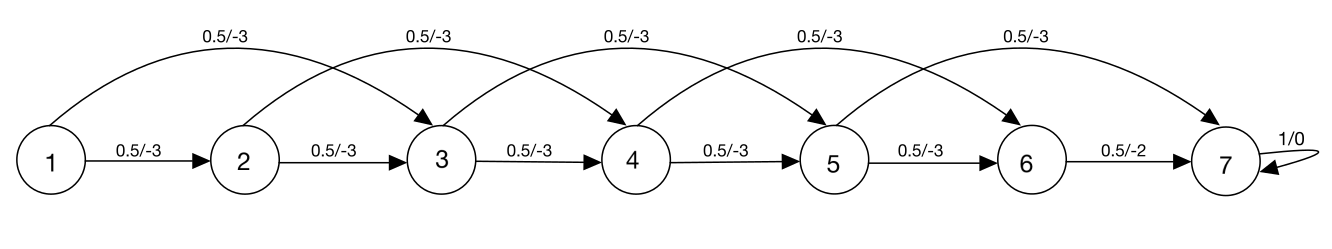}
    }
    \label{app_fig_1}
    \caption{7 state Boyan Chain from \cite{Boyan}}
\end{figure}
\item \textbf{Example-2 (5-State Random Walk):} It consists of a linear arrangement of 5 states with two terminal states. There is a single action at each state. From each state one either moves left or right with equal probability. Moving left from state 1 results in episode termination yielding a reward of 0. Similarly, moving right from state 5 also results in episode termination, however, yielding a reward of +1. The reward associated with all other transitions is 0 and the discount factor $\gamma = 1$. The following figure shows the corresponding MDP.
    \begin{figure}[H]
    \centering
    {
        {\includegraphics[width=\linewidth]{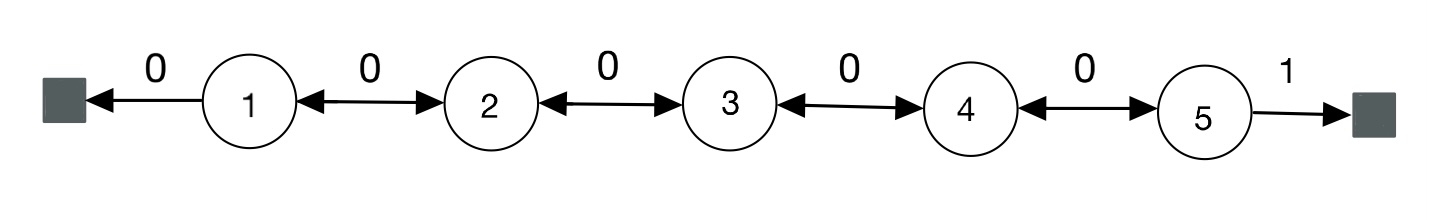}}
    }
    \label{app_fig_2}
    \caption{5-State Random Walk from \cite{FastGradient}}
\end{figure}
\item \textbf{Example-3 (19-State Random Walk):} It consists of a linear arrangement of 19 states. From each state one either moves left or right with equal probability. Moving left from state 1 results in episode termination yielding a reward of -1. Similarly, moving right from state 19 also results in episode termination, however, yielding a reward of +1. The reward associated with all other transitions is 0 and the discount factor $\gamma = 1$. The following figure shows the corresponding MDP:
\begin{figure}[H]
    \centering
    {
        {\includegraphics[width=\linewidth]{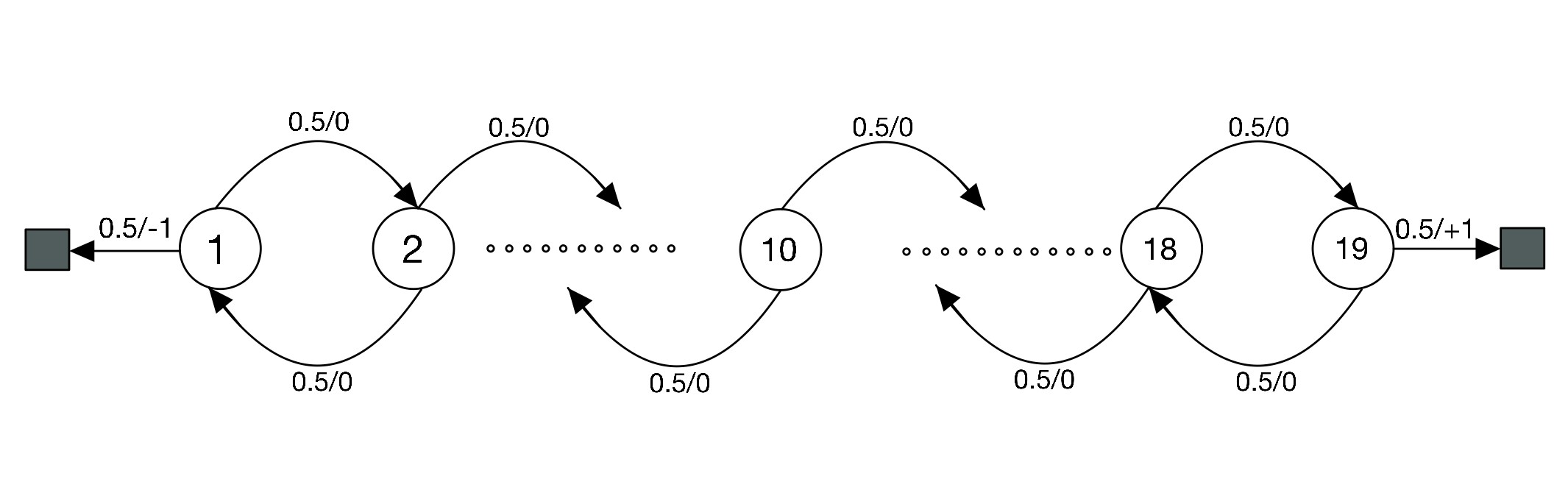}}
    }
    \caption{19 State Random Walk from \cite{SuttonBarto_book}}
    \label{app_fig_3}
\end{figure}
\item \textbf{Example-4 (Random MDP):} This is a randomly generated discrete MDP with 20 states and 5 actions in each state. The transition probabilities are uniformly generated from $[0,1]$ with a small additive constant. The rewards are also uniformly generated from $[0,1]$. The policy and the start state distribution are also generated in a similar way and the discount factor $\gamma = 0.95$. See \cite{Dann} for a more detailed description.
\end{enumerate}
\end{document}